\documentclass[twoside]{article}

\usepackage[accepted]{icml_new/icml2023}
\usepackage{amsmath,amsfonts,bm}

\def\eqref#1{equation~\ref{#1}}

\def\1{\bm{1}}

\DeclareMathAlphabet{\mathsfit}{\encodingdefault}{\sfdefault}{m}{sl}
\SetMathAlphabet{\mathsfit}{bold}{\encodingdefault}{\sfdefault}{bx}{n}

\newcommand{\E}{\mathbb{E}}

\newcommand{\R}{\mathbb{R}}

\usepackage{hyperref}
\usepackage{url}
\usepackage[utf8]{inputenc} 
\usepackage[T1]{fontenc}    
\usepackage{hyperref}       
\usepackage{url}            
\usepackage{booktabs}       
\usepackage{amsfonts}       
\usepackage{nicefrac}       
\usepackage{microtype}      
\usepackage{xcolor}         
\usepackage{array}
\usepackage{xspace}
\usepackage{ICML_new/algorithm}
\usepackage{amsmath,amssymb,amsthm}
\usepackage{bm}
\usepackage{cleveref}
\usepackage{comment}
\usepackage{enumitem}
\usepackage{mathtools}
\usepackage{upgreek}
\usepackage{stackengine} 
\usepackage{autonum}

\newcommand{\calL}{\mathcal{L}}
\newcommand{\calN}{\mathcal{N}}
\newcommand{\calO}{\mathcal{O}}
\newcommand{\calP}{\mathcal{P}}
\newcommand{\calU}{\mathcal{U}}
\newcommand{\calX}{\mathcal{X}}

\newcommand{\nset}{\mathbb{N}}
\newcommand{\nsets}{\mathbb{N}^*}
\newcommand{\setX}{\mathsf{X}}
\newcommand{\pinf}{+\infty}
\newcommand{\iid}{i.i.d.\xspace}

\newcommand{\vMF}{\text{vMF}}

\newcommand{\SW}{\text{SW}}

\newcommand{\swdp}[3]{\text{SW}_p^p({#1},{#2};{#3})}
\newcommand{\swd}[3]{\text{SW}_p({#1},{#2};{#3})}
\newcommand{\wdp}[2]{\text{W}_p^p({#1},{#2})}

\newcommand{\wdone}[2]{\text{W}_1({#1},{#2})}
\newcommand{\swdone}[3]{\text{SW}_1({#1},{#2};{#3})}

\newcommand{\kld}[2]{\text{KL}({#1}||{#2})}

\newcommand{\Sphere}{\mathbb{S}^{d-1}}
\newcommand{\dd}{\mathrm{d}}
\newcommand{\eqsp}{\;}
\newcommand{\ie}{\textit{i.e.}}

\newcommand{\thss}{\theta^*_\sharp}
\newcommand{\ps}[2]{\left\langle#1,#2 \right\rangle}
\newcommand\argsup{\mathrm{arg\,sup}}

\newcommand{\Deltasg}{\Delta_{\text{SG}}}

\newtheorem{theorem}{Theorem}
\crefname{theorem}{theorem}{Theorems}
\Crefname{Theorem}{Theorem}{Theorems}

\newtheorem{proposition}{Proposition}
\crefname{proposition}{proposition}{propositions}
\Crefname{Proposition}{Proposition}{Propositions}

\newtheorem{lemma}{Lemma}
\crefname{lemma}{lemma}{lemmas}
\Crefname{Lemma}{Lemma}{Lemmas}

\newtheorem{corollary}{Corollary}
\crefname{corollary}{corollary}{corollaries}
\Crefname{Corollary}{Corollary}{Corollaries}

\newtheorem{definition}{Definition}
\crefname{definition}{definition}{definitions}
\Crefname{Definition}{Definition}{Definitions}

\newtheorem{remark}{Remark}
\crefname{remark}{remark}{remarks}
\Crefname{Remark}{Remark}{Remarks}

\crefformat{assumption}{{\textbf{A}}#2#1#3}

 

\makeatletter
\newenvironment{equationsize*}[1]{%
  \skip@=\baselineskip 
  #1%
  \baselineskip=\skip@ 
  \equation
}{\nonumber\endequation \ignorespacesafterend} 
\makeatother

 

\makeatletter
\newenvironment{alignsize*}[1]{%
  \skip@=\baselineskip 
  #1%
  \baselineskip=\skip@ 
  \start@align\@ne\st@rredtrue\m@ne
}{\endalign\ignorespacesafterend} 
\makeatother

\usepackage{multirow}
\usepackage{subfigure}
\usepackage{wrapfig}
\renewcommand{\eqref}[1]{(\ref{#1})}

\icmltitlerunning{Shedding a PAC-Bayesian Light on Adaptive Sliced-Wasserstein Distances}

\begin{document}

\twocolumn[
\icmltitle{Shedding a PAC-Bayesian Light on Adaptive Sliced-Wasserstein Distances}



\icmlsetsymbol{equal}{*}

\begin{icmlauthorlist}
\icmlauthor{Ruben Ohana}{equal,yyy}
\icmlauthor{Kimia Nadjahi}{equal,comp}
\icmlauthor{Alain Rakotomamonjy}{sch}
\icmlauthor{Liva Ralaivola}{sch}
\end{icmlauthorlist}

\icmlaffiliation{yyy}{Flatiron Institute, USA}
\icmlaffiliation{comp}{MIT, USA}
\icmlaffiliation{sch}{Criteo AI Lab, France}

\icmlcorrespondingauthor{RO}{rohana@flatironinstitute.org}
\icmlcorrespondingauthor{KN}{knadjahi@mit.edu}

\icmlkeywords{Sliced Optimal Transport, PAC-Bayesian learning, optimization}

\vskip 0.3in
]
\printAffiliationsAndNotice{\icmlEqualContribution}



%


\begin{abstract}
The Sliced-Wasserstein distance (SW) is a computationally efficient and theoretically grounded alternative to the Wasserstein distance. Yet, the literature on its statistical properties -- or, more accurately, its {\em generalization} properties -- with respect to the distribution of slices, beyond the uniform measure, is scarce. To bring new contributions to this line of research, we leverage the PAC-Bayesian theory and a central observation that SW may be interpreted as an average risk, the quantity PAC-Bayesian bounds have been designed to characterize. We provide three types of results: i) PAC-Bayesian generalization bounds that hold on what we refer as {\em adaptive} Sliced-Wasserstein distances, i.e. SW defined with respect to arbitrary distributions of slices (among which data-dependent distributions), ii) a principled procedure to learn the distribution of slices that yields maximally discriminative SW, by optimizing our theoretical bounds, and iii) empirical illustrations of our theoretical findings.
\end{abstract}

\section{Introduction} \label{sec:intro}
The Wasserstein distance is a metric between probability distributions and a key notion of the optimal transport framework \citep{villani2009optimal,peyre2019computational}. Over the past years, it has received a lot of attention from the machine learning community because of its theoretical grounding and the increasing number of problems relying on the computation of distances between measures \citep{solomon2014wasserstein, frogner2015learning, montavon2016wasserstein, kolouri2017optimal, courty2016optimal, schmitz2018wasserstein}, such as the learning of deep generative models \citep{arjovsky2017wasserstein, bousquet2017optimal,tolstikhin2017wasserstein}. As the measures $\mu$ and $\nu$ to be compared are usually unknown, the Wasserstein distance $W(\mu,\nu)$ is estimated through an ``empirical'' version $W(\mu_n,\nu_n)$, where $\mu_n\doteq\{x_1,\ldots, x_n\}$ and $\nu_n\doteq\{y_1,\ldots,y_n\}$ are \iid samples from $\mu$ and $\nu$, respectively (without loss of generality, samples will be assumed to have the same size $n$). Due to its unfavorable $\calO(n^3\log n)$ computational complexity, the Wasserstein distance scales badly on large datasets \citep{peyre2019computational} and alternatives have been devised to overcome this limitation, such as the Sinkhorn algorithm \citep{cuturi2013sinkhorn,cuturi2016smoothed}, multi-scale \citep{oberman2015efficient} or sparse approximations approaches \citep{schmitzer2016sparse}. 

The Sliced-Wasserstein distance (SW) \citep{rabin2012wasserstein} is another computationally efficient alternative, which takes advantage of the closed-form and fast computation of the one-dimensional Wasserstein distance. For $d$-dimensional ($d>1$) samples $\{x_1,\ldots, x_n\}$ and $\{y_1,\ldots, y_n\}$,  the computation of $\SW(\mu_n,\nu_n)$ is done by uniformly sampling $m$ \emph{projection directions} $\{\theta_1,\ldots,\theta_m\}$ and by averaging the $m$ one-dimensional Wasserstein distances $W(\{\langle \theta_j, x_1\rangle,\ldots,\langle \theta_j, x_n\rangle\},\{\langle \theta_j, y_1\rangle,\ldots,\langle \theta_j, y_n\rangle\})$ for $j=1,\ldots,m$. SW has been analyzed theoretically \citep{Bonnotte2013,nadjahi2019asymptotic,bayraktar2019strong,nadjahi2020statistical}, refined to gain additional efficiency \citep{nadjahi2021fast} and to handle ``nonlinear'' projections \citep{kolouri2019generalized, kolouri2020generalized}, and it has been successfully used in a variety of machine learning tasks \citep{bonneel2015sliced, kolouri2016sliced, carriere2017sliced, liutkus2019sliced, deshpande2018generative, kolouri2018sliced, kolouri2019sliced, nadjahi2020approximate, bonet2021sliced, rakotomamonjy2021differentially}.  A direction to yet improve SW consists in adapting $\rho$, the distribution of $\{\theta_i\}_{i=1}^m$ in a data-dependent manner, as done by \emph{maximum SW} (max-SW, \cite{deshpande2019max}), which aims at finding a unique slice $\theta_{\star}$ (or equivalently, the Dirac measure $\delta_{\theta_{\star}}$) that maximizes the one-dimensional Wasserstein distance, or \emph{distributional SW} (DSW) \cite{nguyen2021distributional}, which seeks for a maximally discriminative distribution on the unit sphere. These works fall into the class of what we refer as {\em adaptive Sliced-Wasserstein distances} and denote $\SW(\cdot,\cdot;\rho)$,  overloading the $\SW(\cdot,\cdot)$ notation to make explicit the dependence on $\rho$. 

A question of interest in adaptive SW, which has not been explicitly addressed in previous work, is whether one can learn a distribution $\rho^\star(\mu_n, \nu_n)$ from training data, such that $\swdp{\mu}{\nu}{\rho^\star(\mu_n, \nu_n)}$ is guaranteed to be highly discriminative. In our work, we address this problem by measuring the ``generalization'' gap between $\swdp{\mu_n}{\nu_n}{\rho}$ and $\swdp{\mu}{\nu}{\rho}$. Bounds on this gap can be derived from existing results for max-SW \citep{lin2021,weed2022}.
However, it is unclear how these bounds are able to accommodate distributions $\rho$ that are not reduced to Dirac measures.
To go that direction, we propose the first connection between adaptive SW and \emph{PAC-Bayesian theory} and we derive a novel set of flexible PAC-Bayesian generalization bounds. Our bounds state that with probability $1-\delta$, the following holds for all measures (with non-discrete support) $\rho$ on the $d$-dimensional unit sphere: $\SW(\mu,\nu;\rho)\geq \SW(\mu_n,\nu_n,\rho) - \varepsilon(n,\rho,\delta)$,
where $\varepsilon$ can be written explicitly and captures the properties of $\mu, \nu$, and allows us to control the tightness of the bound via $\rho$.

Three key reasons make the PAC-Bayesian theory \citep{mcallester1999some,catoni07pacbayesian,alquier2021userfriendly} particularly suited to characterize the generalization properties of adaptive SW. First, from a general perspective, the literature shows this framework allows the derivation of tight bounds that can be converted into effective learning procedures \citep{Ambroladze06Tighter,laviolette06pacbayes,Germain09pacbayesian,zantedeschi21stochastic}. Second, PAC-Bayesian bounds deal with the generalization ability of learned distributions; while those distributions usually lie on spaces of predictors, the distributions $\rho$ of interest in our case are the distributions of slices. Lastly, a key quantity of PAC-Bayesian bounds is the \emph{average empirical risk} which, as we will show, can naturally be interpreted as $\swdp{\mu_n}{\nu_n}{\rho}$, our main focus.

The paper is organized as follows. In \Cref{sec:background}, we recall essential notions of Sliced-Wasserstein distances and PAC-Bayesian theory. We then delve into our contributions: \textit{i)} a generic PAC-Bayesian bound for adaptive Sliced-Wasserstein distances and refinements to specific settings (\Cref{sec:pacbound}), \textit{ii)} a theoretically-grounded procedure to train a maximally discriminative Sliced-Wasserstein distances (\Cref{sec:applications})  and \textit{iii)} illustrations of the soundness of our theoretical results through numerical experiments, conducted on both toy and real-world datasets (\Cref{sec:expes}). 

\textbf{Notations.}\; Let $d \in \nsets$ with $\nsets \doteq \nset \backslash \{0\}$. For $x, y \in \R^d$, $\ps{x}{y}$ denotes the dot product between $x$ and $y$, and $\| x \|$ is the Euclidean norm of $x$. For $\mathsf{X}\subset\R^d$, $\calP(\mathsf{X})$ is the set of probability measures supported on $\mathsf{X}$, and $\calP_q(\mathsf{X})$ is the set of probability measures supported on $\mathsf{X}$ with finite moment of order $q$. $\calU(\mathsf{X})$ is the uniform distribution on $\mathsf{X}$, and $\updelta_x$ is the Dirac measure with mass on $x \in \mathsf{X}$.  For $\mu \in \calP(\mathsf{X})$ and $n \in \nsets$, $\mu_n \doteq \frac1n \sum_{i=1}^n \updelta_{x_i}$ is the empirical measure supported on $n$ samples $\{x_1,\ldots,x_n\}$ \iid~from $\mu$. For $\mu \in \calP(\R)$, $F_{\mu}$ is the cumulative distribution function (c.d.f.) of $\mu$ and $F^{-1}_{\mu}$ is its quantile function.

\section{Background}
\label{sec:background}

\subsection{Sliced-Wasserstein Distances} \label{subsec:sw}

Sliced-Wasserstein distances refer to a family of distances between probability measures, which was first introduced by \cite{rabin2012wasserstein} to overcome the computational issues of the Wasserstein distance. We formally define the Wasserstein distance and SW, and explain why the latter can provide significant computational advantages over the former. In what follows, we fix $\mathsf{X} \subset \R^d$.


\begin{definition}[Wasserstein distance]
    Let $p \in [1, \pinf)$. The Wasserstein distance of order $p$ between $\mu, \nu \in \calP(\mathsf{X})$ is
    \begin{equation} \label{eq:wasserstein distance}
        \mathrm{W}_p^p(\mu,\nu) \doteq \underset{\pi\in\Pi(\mu,\nu)}{\inf}\int_{\mathsf{X} \times \mathsf{X}}\|x-y\|^p\dd\pi(x,y) \,,
    \end{equation} 
    where $\Pi(\mu, \nu)\subset\calP(\mathsf{X}\times\mathsf{X})$ denotes the set of probability measures on $\mathsf{X} \times \mathsf{X}$, whose marginals with respect to the first and second variables are $\mu$ and $\nu$ respectively.
\end{definition}

While $\textrm{W}_p$ has been shown to possess appealing theoretical properties, e.g. it is a metric on $\calP_p(\mathsf{X})$ which \emph{metrizes} the weak convergence \citep[Chapter 6]{villani2009optimal}, it is computationally too demanding in general. Indeed, consider two discrete distributions $\mu_n, \nu_n$, each supported on $n$ samples. Computing $\text{W}_p(\mu_n, \nu_n)$ means solving a linear program \citep[Section 3.1]{peyre2019computational}, whose solution is not analytically available in general, but can be approximated with standard solvers from linear programming and combinatorial optimization. However, such methods have a super-cubic cost in practice, and their worst-case computational complexity scales in $\calO(n^3 \log n )$. 

Nevertheless, if $\mu, \nu \in \calP(\R)$, $\text{W}_p(\mu, \nu)$ admits an analytical expression which can be efficiently approximated \citep[Section 2.6]{peyre2019computational}: for any $\mu, \nu \in \calP(\R)$,
\begin{equation} \label{eq:wasserstein_1D_continuous}
    \wdp{\mu}{\nu} = \int_0^1|F^{-1}_\mu (t)-F^{-1}_\nu (t)|^p \dd t \,.
\end{equation}
In particular, for $\mu_n = (1/n)\sum_{i=1}^n \updelta_{x_{i}}$ and $\nu_n = (1/n)\sum_{i=1}^n \updelta_{y_{i}}$ such that, $\forall i \in \{1, \dots, n\}$, $x_i, y_i \in \R$, 
\begin{equation} \label{eq:wasserstein_1D_discrete}
    \wdp{\mu_n}{\nu_n} = \frac1n \sum_{i=1}^n | x_{(i)} - y_{(i)} |^p \,,
\end{equation}
where $x_{(1)} \leq x_{(2)} \leq \dots \leq x_{(n)}$, $y_{(1)} \leq y_{(2)} \leq \dots \leq y_{(n)}$. Computing \eqref{eq:wasserstein_1D_discrete} thus consists in sorting the support points of $\mu_n$ and $\nu_n$, which induces $\calO(n\log n)$ operations.

Sliced-Wasserstein distances leverage the fast computation of $\text{W}_p(\mu, \nu)$ for any $\mu, \nu \in \calP(\R)$ to efficiently compare distributions supported on medium to high-dimensional spaces. Their formal characterization is given below.

\begin{definition}[Sliced-Wasserstein distance] \label{def:sw}
Let $\Sphere \doteq \{\theta \in \R^d~:~\|\theta\| = 1\}$ be the unit sphere in $\R^d$. For $\theta \in \Sphere$, denote by $\theta^* : \R^d \to \R$ the linear map such that for $x \in \R^d$, $\theta^*(x) \doteq \ps{\theta}{x}$. Let $p \in [1, \pinf)$ and $\rho \in \calP(\Sphere)$. The \emph{Sliced-Wasserstein distance of order $p$ based on $\rho$} is defined for $\mu, \nu \in \calP(\mathsf{X})$ as

\begin{equation} \label{eq:def_sw}
    \mathrm{SW}_p^p(\mu,\nu;\rho)
    \doteq \int_{\Sphere} \wdp{\thss\mu}{\thss\nu} \dd\rho(\theta) \,,
\end{equation}
where for any measurable function $f$ and $\xi \in \calP(\R^d)$, $f_\sharp \xi$ is the {\em push-forward measure} of $\xi$ by $f$: for any measurable set $\textsc{A} \subset \R$, $f_\sharp\xi(\textsc{A}) \doteq \xi(f^{-1}(\textsc{A}))$, $f^{-1}(\textsc{A}) \doteq \{x\in\R^d : f(x)\in\textsc{A}\}$.
\end{definition}

\textbf{Computational complexity of SW.}\; By \eqref{eq:def_sw}, $\swdp{\mu}{\nu}{\rho}$ is obtained by computing $\E[\wdp{\thss \mu}{\thss \nu}]$ with $\E$ taken over $\theta \sim \rho$. This expectation is intractable in general, and commonly approximated with the Monte Carlo estimate  \vspace{-4mm}
\begin{equation} \label{eq:swmc_estimate}
    \widehat{\text{SW}}^p_p(\mu, \nu ; \rho) = \frac1m \sum_{j=1}^m \wdp{(\theta_j)^*_\sharp \mu}{(\theta_j)^*_\sharp \nu} \,,
\end{equation}
where $\{\theta_j\}_{j=1}^m$ are \iid~samples from $\rho$. Note that for $\theta \in \Sphere$, $\thss \mu$ and $\thss \nu$ are one-dimensional probability measures, which can be interpreted as projections of $\mu$ and $\nu$ along $\theta$. To illustrate this, consider $\mu_n = (1/n) \sum_{i=1}^n \updelta_{x_i}$ with $x_i \in \R^d$ for $i \in \{1, \dots, n\}$. By definition, $\thss\mu_n = (1/n) \sum_{i=1}^n \updelta_{\ps{\theta}{x_i}}$. Therefore, computing \eqref{eq:swmc_estimate} between $\mu_n$ and $\nu_n$ amounts to projecting $\{x_i\}_{i=1}^n$ and $\{y_i\}_{i=1}^n$ along $\theta_j \sim \rho$, then computing the one-dimensional Wasserstein distance using \eqref{eq:wasserstein_1D_discrete}, for $j \in \{1, \dots, m\}$. This scheme requires $\calO\big(m(dn + n \log n)\big)$ operations which is, in general, faster than computing $\wdp{\mu_n}{\nu_n}$, especially for large $n$.

\textbf{Theoretical properties of SW.}\; Previous works have investigated theoretical properties of $\swdp{\cdot}{\cdot}{\rho}$, to explain its empirical success \citep{Bonnotte2013,bayraktar2019strong,nadjahi2019asymptotic,lin2021,nguyen2021distributional}. However, most results apply to $\rho = \calU(\Sphere)$ only (which corresponds to the original definition of SW, \cite{rabin2012wasserstein}). In particular, whether \eqref{eq:def_sw} is a metric for any $\rho$ has not been established: we show in \Cref{appendix:metric_sw} that $\swdp{\cdot}{\cdot}{\rho}$ is always a pseudo-metric, and we discuss for which choices of $\rho$ it satisfies all metric axioms.

\textbf{Adaptive SW.}\; 
Recent works have argued that the uniform distribution may not be the most relevant choice, depending on the task at hand. Instead, they proposed to learn $\rho$ from the observed data. This strategy provides $\swdp{\cdot}{\cdot}{\rho}$ with an actual degree of freedom $\rho$, and motivates the term {\em adaptive} Sliced-Wasserstein distance. Specifically, \cite{deshpande2019max} and \cite{nguyen2021distributional} solve a tailored optimization problem in $\rho$ targetting a high discriminative power of $\rho$, in the sense that $\rho$ puts more mass on the $\theta \in \Sphere$ that maximize the separation of $\thss\mu$ and $\thss\nu$. The \emph{maximum Sliced-Wasserstein distance} (max-SW, \cite{deshpande2019max}) is defined as  \vspace{-1mm}
\begin{align}
   \text{maxSW}(\mu,\nu) &\doteq \swdp{\mu}{\nu}{\rho^\star_{\text{maxSW}}(\mu,\nu)} \label{eq:maxsw} \\  
   \text{with \quad} \rho^\star_{\text{maxSW}}(\mu,\nu) &\doteq \underset{\updelta_\theta:\,\theta \in \Sphere}{\argsup} \swdp{\mu}{\nu}{\updelta_\theta}\,, \label{eq:argsup_maxsw}
\end{align}
while the \emph{distributional Sliced-Wasserstein distance} (DSW, \cite{nguyen2021distributional}) is given by
\begin{align}
   \text{DSW}(\mu,\nu) &\doteq \swdp{\mu}{\nu}{\rho^\star_{\text{DSW}}(\mu,\nu)} \label{eq:dsw} \\
   \text{with \quad} \rho^\star_{\text{DSW}}(\mu,\nu)&\doteq\underset{\underset{\E_{\theta,\theta'\sim\rho}|\theta^\top \theta'| \leq C}{\rho \in \calP(\Sphere),} }{\argsup} \swdp{\mu}{\nu}{\rho} \label{eq:argsup_dsw}
\end{align}
\vspace{-1.5em}

where in~\eqref{eq:argsup_dsw}, $\theta$ and $\theta'$ are independent and $C > 0$ is a hyperparameter. We have decoupled the search for the maximizing distances~\eqref{eq:maxsw},\eqref{eq:dsw} and the maximum arguments~\eqref{eq:argsup_maxsw},\eqref{eq:argsup_dsw} for reasons we clarify below.

While there exist statistical guarantees on the gap between $\text{maxSW}(\mu,\nu)$ and $\text{maxSW}(\mu_n,\nu_n)$ \citep{lin2021,weed2022} (or between $\text{DSW}(\mu,\nu)$ and $\text{DSW}(\mu_n,\nu_n)$ \cite{nguyen2021distributional}), there is no explicit theoretical argument on the error entailed by the learned distribution $\rho^\star_{\text{maxSW}}(\mu_n,\nu_n)$ (or $\rho^\star_{\text{DSW}}(\mu_n,\nu_n)$) considered on its own, outside the optimization procedure of max-SW (or DSW).
Given new samples $\{x_1',\ldots,x_n'\}$ and $\{y_1',\ldots,y_n'\}$ from $\mu$ and $\nu$, with empirical distributions $\mu'_n$ and $\nu'_n$, there is no guarantee for $\swdp{\mu'_n}{\nu'_n}{\rho^\star_{\text{maxSW}}(\mu_n,\nu_n)}$ to be high, or in other words, there is no argument ensuring the discriminative power of $\rho^\star_{\text{maxSW}}(\mu_n,\nu_n)$. 
One way to palliate this lack of theory and to go one step further than the max-SW and DSW cases, is to derive general results relating $\swdp{\mu_n}{\nu_n}{\rho}$ and $\swdp{\mu}{\nu}{\rho}$, for families of distributions $\rho\in\calP(\Sphere)$. This is what we bring in the present work in the form of a generalization bound rooted in the PAC-Bayesian theory.

\subsection{PAC-Bayesian Theory} \label{subsec:background_pacbayes}

PAC-Bayesian theory aims at assessing the ability of learning algorithms to generalize to unseen data, by deriving \emph{generalization bounds}. 
Let $\calX \subset \R^q$, $q \in \nsets$, and $S_n \doteq \{z_i\}_{i=1}^n$ a dataset of \iid~samples from an unknown probability measure $\xi \in \calP(\calX)$. Consider a learning algorithm whose outputs depend on the training data $S_n$ and a vector of parameters $\omega \in \Omega$. The performance of such algorithm can be assessed via a {\em loss function} $\ell : \Omega \times \calX \to \R_+$. Fix $\omega \in \Omega$. The \emph{empirical $\ell$-risk} $\hat{r}_\ell(\omega, S_n)$ and \emph{true $\ell$-risk} $r_\ell(\omega)$ are defined as, \vspace{-3mm}
\begin{align}
    \hat{r}_\ell(\omega, S_n) &\doteq \frac1n \sum_{i=1}^n \ell(\omega, z_i) \label{eq:empiricalrisk} \\
    r_\ell(\omega) &\doteq \E_{z \sim \xi}[\ell(\omega, z)] \label{eq:theoreticalrisk}
\end{align}
A key objective of a learning procedure is to optimize (e.g. minimize) the true risk \eqref{eq:theoreticalrisk}, which in practice cannot be achieved, because $\xi$ is unknown. Instead, one focuses on optimizing \eqref{eq:empiricalrisk} over $\omega \in \Omega$, a sound strategy provided the minimizer of \eqref{eq:empiricalrisk} accurately estimates the minimizer of \eqref{eq:theoreticalrisk}: this can be assessed via PAC-Bayesian bounds. 

Let $\rho \in \calP(\Omega)$. PAC-Bayesian theory analyzes the generalization ability of $\rho$ by measuring the gap between the \emph{average empirical $\ell$-risk} $\E_{\omega\sim\rho}[\hat{r}_\ell(\omega, S_n)]$ and the \emph{average true $\ell$-risk} $\E_{\omega\sim\rho}[r_\ell(\omega)]$. A classical PAC-Bayesian bound was derived by \cite{catoni03pacbayesian} and is recalled below.

\begin{theorem}[\cite{catoni03pacbayesian}]
\label{th:catoni}
Let $\rho_0 \in\calP(\Omega)$ be a prior distribution. Assume that $0 \leq \ell \leq C$. For all $\lambda>0$, for any $\delta\in(0,1)$, the following holds with probability at least $1 - \delta$ (over the draw of the dataset $S_n$): $\forall \rho\in\mathcal{P}(\Omega)$,
\begin{align}
    &\E_{\omega \sim \rho}[r_{\ell}(\omega)] \label{eq:catoni bound}\\
    &\leq \E_{\omega \sim \rho}[\hat{r}_{\ell}(\omega, S_n)] + \frac{\lambda C^2}{8n} + \frac1{\lambda} \left\{ \kld{\rho}{\rho_0} +\log\frac{1}{\delta} \right\} \,, 
\end{align}
where $\kld{\rho}{\rho_0}$ is the \emph{Kullback-Leibler divergence} between $\rho$ and $\rho_0$\,: if $\rho$ is absolutely continuous with respect to $\rho_0$, $\kld{\rho}{\rho_0} \doteq \int\log\left(\rho(\dd \theta) / \rho_0(\dd\theta) \right)\rho(\dd\theta)$.
\end{theorem}

The literature on PAC-Bayes is rich of many other bounds, and we refer to \citep{alquier2021userfriendly} for an extensive survey. In our work, we focus on Catoni's bound because it is generic (appropriate settings of $\lambda$ give rise to other well-known bounds) as are the proof techniques used to derive it \citep[Section 2]{alquier2021userfriendly}. 

\textbf{Applications.}\; PAC-Bayesian bounds allow to control the true risk via a function depending on the empirical risk. For example, minimizing the left-hand side term of Catoni's bound \eqref{eq:catoni bound} over $\rho \in \calP(\Omega)$ yields a data-dependent distribution which guarantees the highest generalization ability \citep[Section 2.1.2]{alquier2021userfriendly}. PAC-Bayesian theory was also applied for specific tasks, e.g. classification \citep{mcallester1999some}, ranking \citep{ralaivola2010}, density estimation \citep{higgs2010}, deep learning \citep{dziugaite2017,cheriefabdellatif2022}.

\section{Generalization Bounds for Adaptive Sliced-Wasserstein Distances} \label{sec:pacbound}

In this section, we leverage the PAC-Bayesian framework to derive generalization bounds for adaptive Sliced-Wasserstein distances. Proofs are deferred to \Cref{appendix:postponed_proofs}.

Before presenting our main results, we clarify the notion of generalization for adaptive SW. In practice, since one generally has access to data generated from unknown probability measures $\mu, \nu$, empirical estimates $\swdp{\mu_n}{\nu_n}{\rho}$ are computed instead of $\swdp{\mu}{\nu}{\rho}$. Besides, adaptive SW means that an algorithm is deployed to learn $\rho$ from $\mu_n, \nu_n$, so that $\swdp{\mu_n}{\nu_n}{\rho}$ is sufficiently discriminative (\Cref{subsec:sw}). In this context, the learning algorithm is said to generalize well if the distribution learned from $\mu_n, \nu_n$ (denoted by $\rho(\mu_n, \nu_n)$) is such that $\swdp{\cdot}{\cdot}{\rho(\mu_n, \nu_n)}$ is discriminative, even on unseen data. More formally, given new samples $\{x_1',\ldots,x_n'\}$ and $\{y_1',\ldots,y_n'\}$ from $\mu$ and $\nu$, with associated empirical distributions $\mu'_n$ and $\nu'_n$, $\swdp{\mu'_n}{\nu'_n}{\rho(\mu_n,\nu_n)}$ should be large.

Therefore, we measure generalization as the gap between $\swdp{\mu}{\nu}{\rho}$ and $\swdp{\mu_n}{\nu_n}{\rho}$ for any $\rho \in \calP(\Sphere)$. We first derive a general bound on this gap, using PAC-Bayesian theory, then refine it to specific settings directed by conditions on the supports and the moments of $\mu$ and $\nu$.

\begin{table}[t!]
    \centering
    \renewcommand{\arraystretch}{1.2}
    \begin{tabular}{cc}
        \toprule
        \textbf{PAC-Bayes framework} & \textbf{Our framework} \\
        \hline \\[-4mm]
        $\{z_i\}_{i=1}^n$ & $\{(x_i, y_i)\}_{i=1}^n$ \\
        $\xi \in \calP(\calX)$ & $\mu \times \nu \in \calP(\R^d) \times \calP(\R^d)$ \\
        $\omega \in \Omega$ & $\theta \in \Sphere$ \\
        $\hat{r}_\ell(\omega, \{z_i\}_{i=1}^n)$ & $\wdp{\thss \mu_n}{\thss \nu_n}$ \\
        $\E_{\omega \sim \rho}[\hat{r}_\ell(\omega, \{z_i\}_{i=1}^n)]$ & $\swdp{\mu_n}{\nu_n}{\rho}$ \\
        $r_\ell(\omega)$ & $\E_{(x_i, y_i)_{i=1}^n}[\wdp{\thss \mu_n}{\thss \nu_n}]$ \\
        $\E_{\omega \sim \rho}[r_\ell(\omega)]$ & $\E_{(x_i, y_i)_{i=1}^n}[\swdp{\mu_n}{\nu_n}{\rho}]$ \\[1mm]
        \toprule
    \end{tabular}
    \vspace{-5mm}
    \caption{Analogy between PAC-Bayes theory and our work.}
    \label{table:analogy}
\end{table}
\subsection{A Generic Generalization Bound} \label{subsec:general_bound}

We establish a first generalization bound for adaptive SW, by combining statistical properties of adaptive SW and techniques from PAC-Bayesian theory.

\begin{theorem} \label{thm:generic_bound}
    Let $p \in [1, \pinf)$ and $\mu, \nu \in \calP_p(\R^d)$. Assume there exists a constant $\varphi_{\mu, \nu, p}$, possibly depending on $\mu, \nu$ and $p$ such that: $\forall \lambda > 0$, $\forall \theta \in \Sphere$,
    \begin{align}
        &\E \left[ \exp \left( \lambda \big\{ \mathrm{W}_p^p(\thss \mu_n, \thss \nu_n) - \E[\mathrm{W}_p^p(\thss \mu_n, \thss \nu_n)]
        \big\} \right) \right] \\
        &\leq \exp(\lambda^2 \varphi_{\mu, \nu, p}\,n^{-1}) \eqsp,  \label{eq:momentgenfn}
    \end{align}
    where $\E$ is taken with respect to the support points of $\mu_n$ and $\nu_n$. Additionally, assume there exists $\psi_{\mu, \nu, p} : \nsets \to \R_+$, possibly depending on $\mu, \nu$ and $p$, such that, $\forall \rho \in \calP(\Sphere)$,
    \begin{equation} \label{eq:sample_cpx_sw}
        \E \big| \mathrm{SW}_p^p(\mu_n,\nu_n;\rho)- \mathrm{SW}_p^p(\mu,\nu;\rho) 
        \big| \leq \psi_{\mu, \nu, p}(n) \,.
    \end{equation}
    Let $\rho_0 \in \calP(\Sphere)$. Then, for any $\delta \in (0, 1)$, the following holds with probability at least $1 - \delta$: $\forall \rho \in \calP(\Sphere)$,
    \begin{align}
    \mathrm{SW}_p^p(\mu,\nu;\rho) &\geq
    \mathrm{SW}_p^p(\mu_n,\nu_n;\rho) - \frac{\lambda}{n} \varphi_{\mu, \nu, p}\\
    - \frac1{\lambda} \Big\{ &\kld{\rho}{\rho_0} + \log\Big(\frac1{\delta}\Big)\Big\} - \psi_{\mu, \nu, p}(n) \,. \label{eq:generic_bound}
    \end{align}
\end{theorem}

\begin{table*}[t]
    \centering
    \begin{tabular}{cccc}
        \toprule
        & \multirow{2}{*}{Bounded supports} & \multicolumn{2}{c}{Unbounded supports} \\
        \cmidrule(lr){3-4}
        & & Sub-Gaussianity (Def.~\ref{def:subg}) & Bernstein moments (Def.~\ref{def:bernstein_cond}) \\
        \hline
        $\varphi_{\mu, \nu, p}$& \Cref{prop:momgenfn_compact} & \Cref{prop:momgenfn_subg} & \Cref{prop:momgenfn_bernstein} \\
        $\psi_{\mu, \nu, p} $& \Cref{prop:splcpx_compact} & \multicolumn{2}{c}{\cite{manole2020minimax}} \\
        \toprule
    \end{tabular}
    \caption{Overview of the explicit forms of $\varphi_{\mu, \nu, p}$ and $\psi_{\mu, \nu, p}$ under different assumptions.}
    \label{table: assumptions values}
\end{table*}

\textbf{Link with PAC-Bayesian theory.}\; \Cref{thm:generic_bound} can be interpreted as a novel PAC-Bayesian bound tailored to adaptive SW: the formal analogy between classical PAC-Bayesian framework and our work is summarized in \Cref{table:analogy}. The key element is that $\mathrm{W}_p^p(\thss \mu_n, \thss \nu_n)$ for some $\theta\in\Sphere$ can be seen as an empirical risk \eqref{eq:empiricalrisk}, and consequently, the average empirical risk is exactly $\swdp{\mu_n}{\nu_n}{\rho}$ \eqref{eq:def_sw}. Nevertheless, we emphasize that \Cref{thm:generic_bound} is not obtained by simply replacing the risks in \Cref{th:catoni} according to \Cref{table:analogy}. Indeed, this would return an \emph{upper} bound (in terms of $\swdp{\mu_n}{\nu_n}{\rho}$) for $\E[\swdp{\mu_n}{\nu_n}{\rho}]$, while we propose a \emph{lower} bound for $\swdp{\mu}{\nu}{\rho}$ \eqref{eq:generic_bound}. Instead, we propose the following slight paradigm shift: while classical PAC-Bayesian theory aims at \emph{minimizing} the average true risk (hence, the \emph{upper} bounds), our goal is to \emph{maximize} $\swdp{\mu}{\nu}{\rho}$ over $\rho$ (hence, the need of \emph{lower} bounds). Therefore, \Cref{thm:generic_bound} is established by first, adapting the elements of proof of \Cref{th:catoni} to establish a lower-bound for $\E[\swdp{\mu_n}{\nu_n}{\rho}]$, then bounding from above $\E[\swdp{\mu_n}{\nu_n}{\rho}]$ by $\swdp{\mu}{\nu}{\rho}$.

\textbf{Discussion.}\; Since our bound holds for all $\rho \in \calP(\Sphere)$ \eqref{eq:generic_bound}, it is therefore valid for $\rho_{\text{maxSW}}^{\star}$~\eqref{eq:argsup_maxsw} and $\rho_{\text{DSW}}^{\star}$~\eqref{eq:argsup_dsw} computed by max-SW and DSW. However, our bound is vacuous for max-SW, because the KL penalty term is evaluated on a Dirac distribution. In that singular case, more informative generalization bounds can be deduced, using \citep{lin2021,weed2022} instead of PAC-Bayes: we elaborate on this in \Cref{subsec:proof_genbound}.

\looseness =-1
We now clarify the role of each term involved in \eqref{eq:generic_bound}. The KL divergence and $\lambda$ arise from adapting the proof techniques of Catoni's bound, so their influence on the generalization gap can be further illustrated with the examples in \citep[Section 2.1.3]{alquier2021userfriendly}. 
More precisely, the KL divergence results from a change of measure inequality known as Donsker-Varadhan's lemma \citep{donsker1975variational}. Previous work have applied other change of measure inequalities to derive PAC-Bayesian bounds in terms of other divergences than KL \citep{alquier2018simpler}. Nevertheless, standard PAC-Bayesian bounds rely on the use of Donsker-Varadhan's lemma, hence the KL divergence. As we introduce the first connection between PAC-Bayesian theory and SW, we decided to use the most common technique.

Then, the quantities $\varphi_{\mu, \nu, p}$ and $\psi_{\mu, \nu, p}(n)$ capture the properties of SW, $\mu$ and $\nu$. More precisely, $\varphi_{\mu, \nu, p}$ bounds the moment-generating function of a centered version of $\wdp{\thss\mu_n}{\thss\nu_n}$ for any $\theta \in \Sphere$, while $\psi_{\mu, \nu, p}$ reflects the \emph{sample complexity} of $\swdp{\cdot}{\cdot}{\rho}$ for any $\rho \in \calP(\Sphere)$. To further illustrate this, we specialize our generic bound \eqref{eq:generic_bound} under different settings: $\mu, \nu$ have bounded supports, are sub-Gaussian, or satisfy a Bernstein-type moment condition. We present our results in the next sections and summarize them in \Cref{table: assumptions values}. 

\subsection{Application to Measures with Bounded Support} \label{sec:a1}

We first consider distributions supported on a bounded domain. We derive $\varphi_{\mu, \nu, p}$ by applying similar arguments as in the proof of McDiarmid's inequality \citep{mcdiarmid1989method}, similarly to \citep[Proposition 20]{weedbach2019}.

\begin{proposition} \label{prop:momgenfn_compact}
Let $\setX \subset \R^d$ such that $\setX$ has a finite  diameter $\Delta$, \ie~$\Delta \doteq \sup_{(x,x') \in \setX^2} \|x-x'\| < \pinf$. Let $p \in [1, \pinf)$, $\mu, \nu \in \calP(\setX)$. Then, $\mu, \nu \in \calP_p(\setX)$ and $\varphi_{\mu, \nu, p} = \Delta^{2p} / 2$. 
\end{proposition}

Next, we adapt the proof of \citep[Lemma B.3]{manole2020minimax} to compute the explicit form of $\psi_{\mu, \nu, p}$ in this setting.

\begin{proposition} \label{prop:splcpx_compact}
    Let $\mu, \nu \in \calP(\setX)$, where $\setX \subset \R^d$ has a finite  diameter $\Delta$. 
    Let $p \in [1, \pinf)$. Then, there exists a constant $C$ such that, $\psi_{\mu, \nu,p}(n) = C p \Delta^p n^{-1/2}$.
\end{proposition}

By combining \Cref{prop:momgenfn_compact,prop:splcpx_compact}, we refine \Cref{thm:generic_bound} to distributions supported on a bounded domain: the resulting bound is given in \Cref{secapp: finalbound bounded}.

\subsection{Application to Sub-Gaussian Measures} \label{sec:a2_subg}

Next, we apply \Cref{thm:generic_bound} to distributions with unbounded supports. To handle this case, we assume specific constraints on the moments on $\mu, \nu$, then derive $\varphi_{\mu, \nu, p}$ by using generalizations of McDiarmid's inequalities. 
More precisely, we assume that $\mu, \nu$ are sub-Gaussian distributions.

\begin{definition}[Sub-Gaussian distribution] \label{def:subg}
    Let $\mu \in \calP(\R^d)$ and $\sigma > 0$. $\mu$ is a \emph{sub-Gaussian distribution} with \emph{variance proxy} $\sigma^2$ if the following holds: for any $\theta \in \Sphere$, for $\lambda \in \R$, $\int_\R \exp(\lambda t) \dd(\thss \mu)(t) \leq \exp(\lambda^2 \sigma^2 / 2)$.
\end{definition}

The next proposition results from applying the generalized McDiarmid's inequality for unbounded spaces with finite \emph{sub-Gaussian diameter} \citep{kontorovicha14} (\Cref{appendix:momgen_subg}). 

\begin{proposition} \label{prop:momgenfn_subg}
    Let $\mu, \nu \in \calP(\R^d)$ such that $\mu, \nu$ are sub-Gaussian with variance proxy $\sigma^2, \tau^2$ respectively. Then, $\mu, \nu \in \calP_1(\R^d)$ and $\varphi_{\mu, \nu, 1} = \sigma^2 + \tau^2$.
\end{proposition}

The last ingredient to specialize \Cref{thm:generic_bound} is to derive $\psi_{\mu, \nu, p}$ for $\mu, \nu$ satisfying either \Cref{def:subg}. To this end, we leverage the rate recently established in \citep[Theorem 2]{manole2020minimax}, which shows that $\psi_{\mu, \nu, p}$ scales as $n^{-1/2}\log(n)$ if $\mu, \nu$ are sub-Gaussian distributions. Our final bound is obtained by plugging \Cref{prop:momgenfn_subg} and the explicit formula of $\psi_{\mu, \nu, 1}$ in \Cref{thm:generic_bound}. We present this result and its detailed proof in \Cref{secapp: final bound unbounded}.

\subsection{Bound for Measures with Bernstein moment conditions} \label{subsec:bernstein}

We study a more general class of distributions: we consider the \emph{Bernstein-type moment condition} below, which is milder than sub-Gaussian distributions.

\begin{definition}[Bernstein condition] \label{def:bernstein_cond}
    Let $\mu \in \calP(\R^d)$ and $\sigma^2, b > 0$. $\mu$ is said to satisfy the \emph{$(\sigma^2, b)$-Bernstein condition} if for any $k \in \nset$, $k \geq 2$, for any $\theta \in \Sphere$, $
    \int_{\R} | t |^k \dd(\thss \mu)(t) \leq \sigma^2 k! b^{k-2} / 2$.
\end{definition}

\Cref{def:subg} is strictly stronger than \Cref{def:bernstein_cond}: if $\mu \in \calP(\R^d)$ verifies the $(\sigma^2, b)$-Bernstein condition, then $\mu$ belongs to the class of heavy-tailed distributions called \emph{sub-exponential distributions} \citep{embrechts2013modelling}, which contains sub-Gaussian distributions. Hence, the class of sub-Gaussian distributions is smaller than the class of distributions characterized by \Cref{def:bernstein_cond}. 

Consider $\mu, \nu$ satisfying \Cref{def:bernstein_cond}. First, we leverage \citep[Theorem 2]{manole2020minimax} in that setting again, to show that $\psi_{\mu, \nu, p}$ scales as $n^{-1/2} \log(n)$ (\Cref{secapp: final bound unbounded}). Then, we apply the Bernstein-type McDiarmid's inequality given in \citep[Theorem 5.1]{lei2020} to establish \Cref{prop:momgenfn_bernstein}.

\begin{proposition} \label{prop:momgenfn_bernstein} 
    Let $\mu, \nu \in \calP(\R^d)$ be two distributions satisfying the Bernstein condition with parameters $(\sigma^2, b)$ and $(\tau^2, c)$ respectively. Let $\sigma^2_\star = \max(\sigma^2, \tau^2)$, $b_\star = \max(b, c)$. Then, $\mu, \nu \in \calP_1(\R^d)$ and, for any $\lambda \in \R_+^*$ s.t. $\lambda < n/(2b_\star)$, 
    \begin{align}
        &\E \left[ \exp \left( \lambda \big\{ \mathrm{W}_1(\thss \mu_n, \thss \nu_n) - \E[\mathrm{W}_1(\thss \mu_n, \thss \nu_n)]
        \big\} \right) \right] \\
        &\leq \exp(\lambda^2 \varphi_{\mu, \nu, 1}(\lambda, n)n^{-1}) \eqsp, \label{eq:momentgenfn_bernstein}
    \end{align}
    where $\varphi_{\mu, \nu, 1}(\lambda, n) = 2 \sigma_\star^2 n^{-1} (1 - 2 b_\star \lambda n^{-1})^{-1}$.
\end{proposition}

We emphasize the following difference between equations~\eqref{eq:momentgenfn_bernstein} and \eqref{eq:momentgenfn}: $\varphi_{\mu, \nu, 1}$ is a function of $\lambda \in \Lambda \subset \R_+$ and $n \in \nsets$, while in \Cref{thm:generic_bound}, $\varphi_{\mu, \nu, p}$ is assumed to be a constant. Nevertheless, the proof of \Cref{thm:generic_bound} can easily be adapted to derive a generic generalization bound assuming $\varphi_{\mu, \nu, p}$ depends on $(\lambda, n)$: we give the corresponding statement in \Cref{thm:extension_generic_bound}. Hence, by plugging \Cref{prop:momgenfn_bernstein} and \citep[Theorem 2]{manole2020minimax} in \Cref{thm:extension_generic_bound}, we derive the generalization bound for distributions under the Bernstein moment condition: see \Cref{secapp: final bound unbounded}.

Note that for $\mu, \nu$ satisfying \Cref{def:subg} or \Cref{def:bernstein_cond}, we derived $\varphi_{\mu, \nu, p}$ for $p = 1$ only: the generalized McDiarmid's inequalities used in the proofs of \Cref{prop:momgenfn_subg,prop:momgenfn_bernstein} can be applied if $\text{W}_p^p$ is Lipschitz \citep{kontorovicha14,lei2020}. This property is easily verified for $p = 1$, but not for $p > 1$. Hence, the derivation of $\varphi_{\mu, \nu, p}$ for $p > 1$ for such types of distributions with unbounded domains requires different proof techniques. We leave this problem for future work. 

\begin{algorithm}[t]
   \caption{PAC-SW: Adaptive SW via PAC-Bayes bound optimization.}
   \label{alg:PACSW-algorithm}
\begin{algorithmic}
    \STATE \textbf{Input:} dataset $\{(x_i,y_i)\}_{i=1}^n$, parameter $\lambda$, prior $\rho_0$, initialization $\rho^{(0)}$, number of iterations $T$, learning rate $\eta$
    \FOR{$t \gets 1$ to $T$}
    \STATE
    $\mathcal{L}(\rho^{(t-1)})\hspace{-1mm}= \swdp{\mu_n}{\nu_n}{\rho^{(t-1)}} - \kld{\rho^{(t-1)}}{\rho_0}/ \lambda $
    \STATE $\rho^{(t)}= \rho^{(t-1)} + \eta\nabla_{\rho} \mathcal{L}(\rho^{(t-1)})$
    \ENDFOR
    \STATE \textbf{return} $\rho^{(T)}$
\end{algorithmic}
\end{algorithm}

\section{Optimization of Generalization Bounds for Adaptive SW} \label{sec:applications}

We develop a principled methodology to learn a highly discriminative Sliced-Wasserstein distance, by optimizing our PAC-Bayesian generalization bounds. The idea consists in making the lower bounds of $\swdp{\mu}{\nu}{\rho}$ derived in \Cref{sec:pacbound} as tight as possible, in order to increase $\swdp{\mu}{\nu}{\rho}$ while attaining a small generalization gap. 

Given a training dataset $\{(x_i, y_i)\}_{i=1}^n$ and a prior $\rho_0 \in \calP(\Sphere)$, our objective is to find $\rho^\star(\mu_n, \nu_n)$ such that,
\begin{equation} \label{eq:rho_star}
    \rho^\star(\mu_n, \nu_n) = \underset{\rho \in \mathcal{F}}{\argsup}~\swdp{\mu_n}{\nu_n}{\rho} - \frac{\kld{\rho}{\rho_0}}{\lambda} \;
\end{equation}
with $\mathcal{F}$ a family of probability measures supported on $\Sphere$. The choice of $\mathcal{F}$ manages the complexity of solving \eqref{eq:rho_star}: it should allow simple optimization, while being flexible to make $\rho^\star(\mu_n, \nu_n)$ expressive enough.
We first propose to parameterize $\mathcal{F}$ as the class of \emph{von Mises-Fisher distributions}.
\begin{figure*}[t!]
    \centering
    \subfigure[{$\mu = \nu = \calU([0, 5]^d)$}]{
    \includegraphics[width=.235\linewidth]{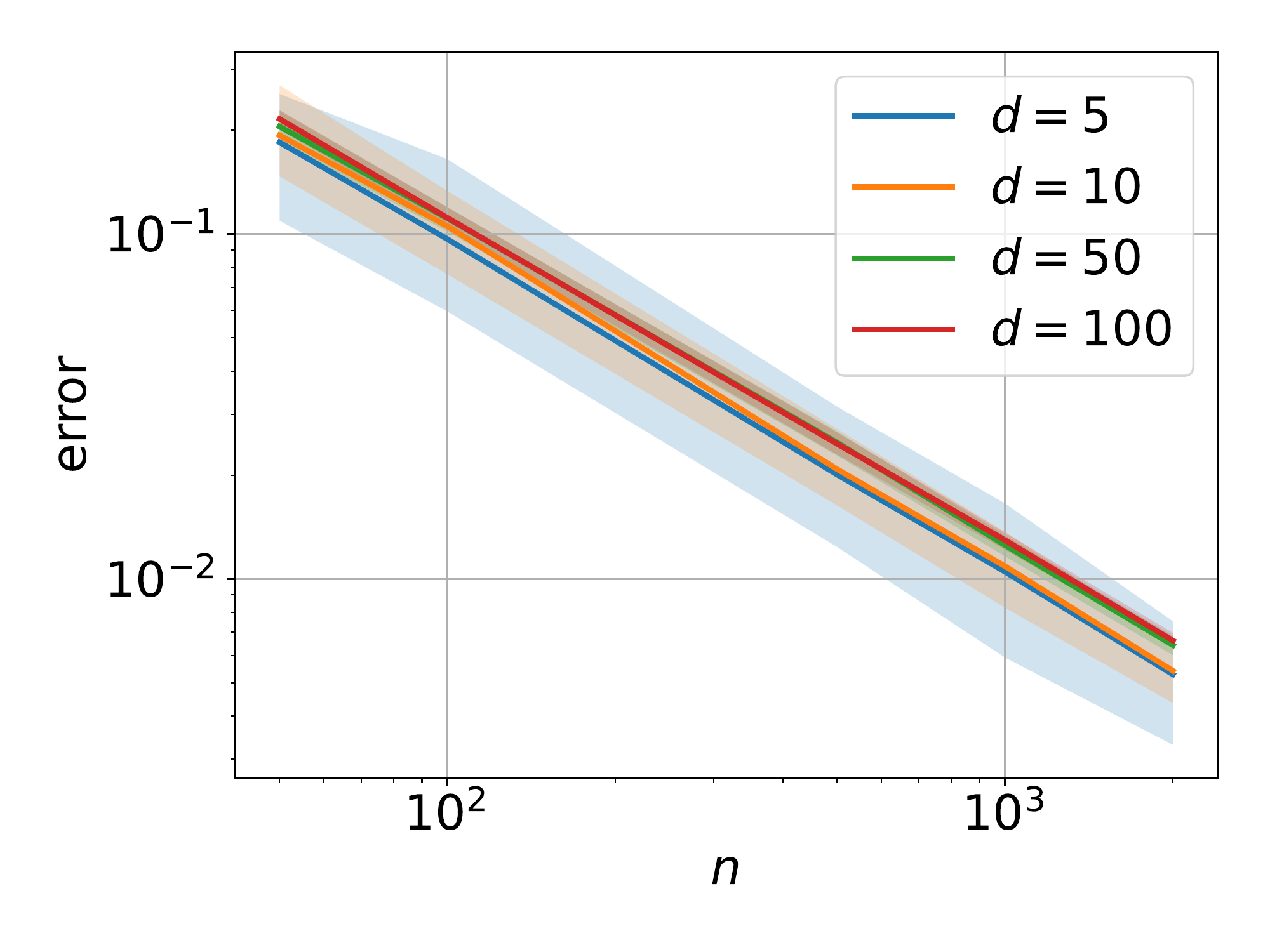}
    \includegraphics[width=.235\linewidth]{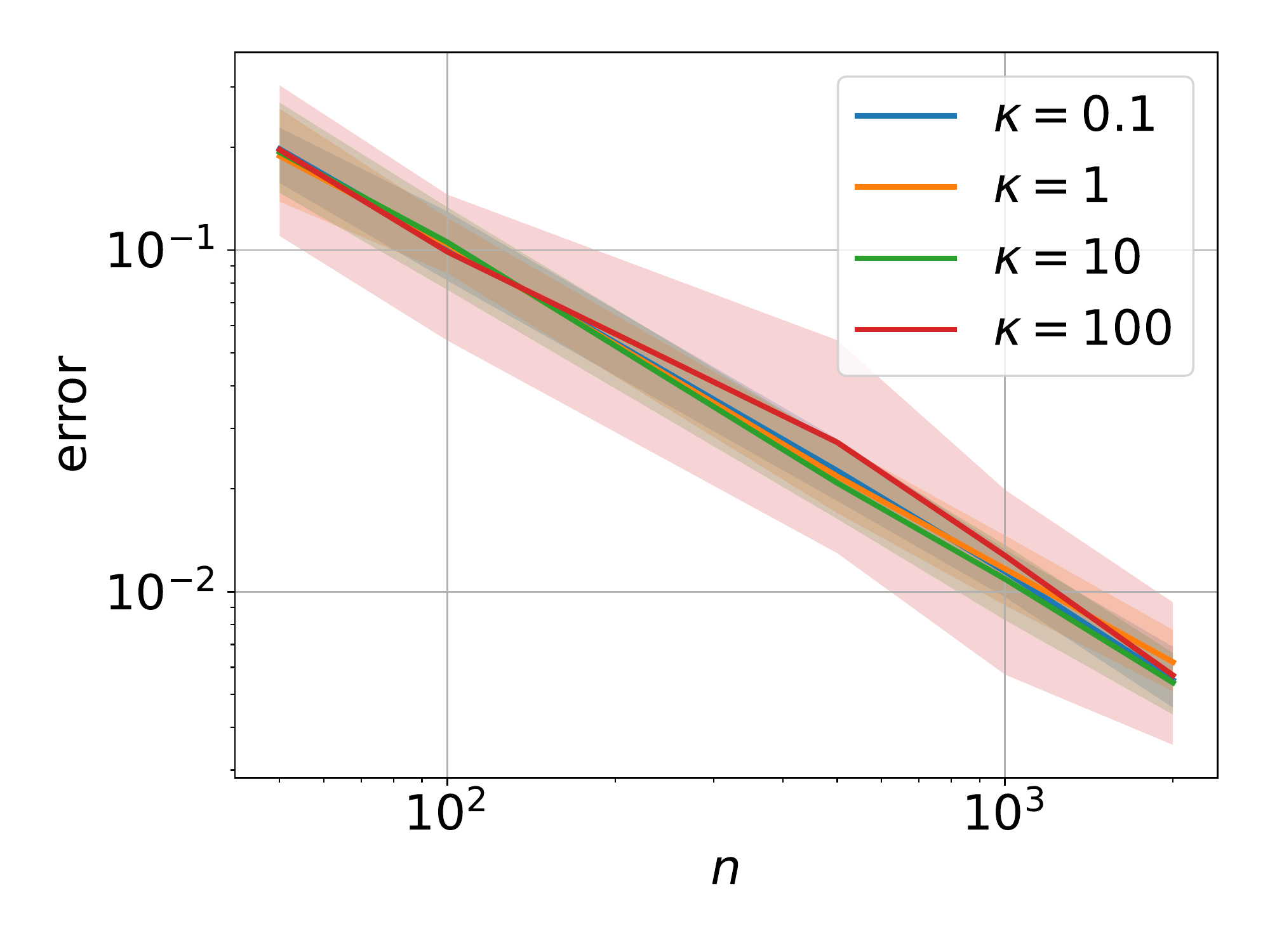}
    \label{subfig:uniform}}
    \subfigure[$\mu = \nu = \calN({\bf0}, \Sigma_d)$]{
    \includegraphics[width=.235\linewidth]{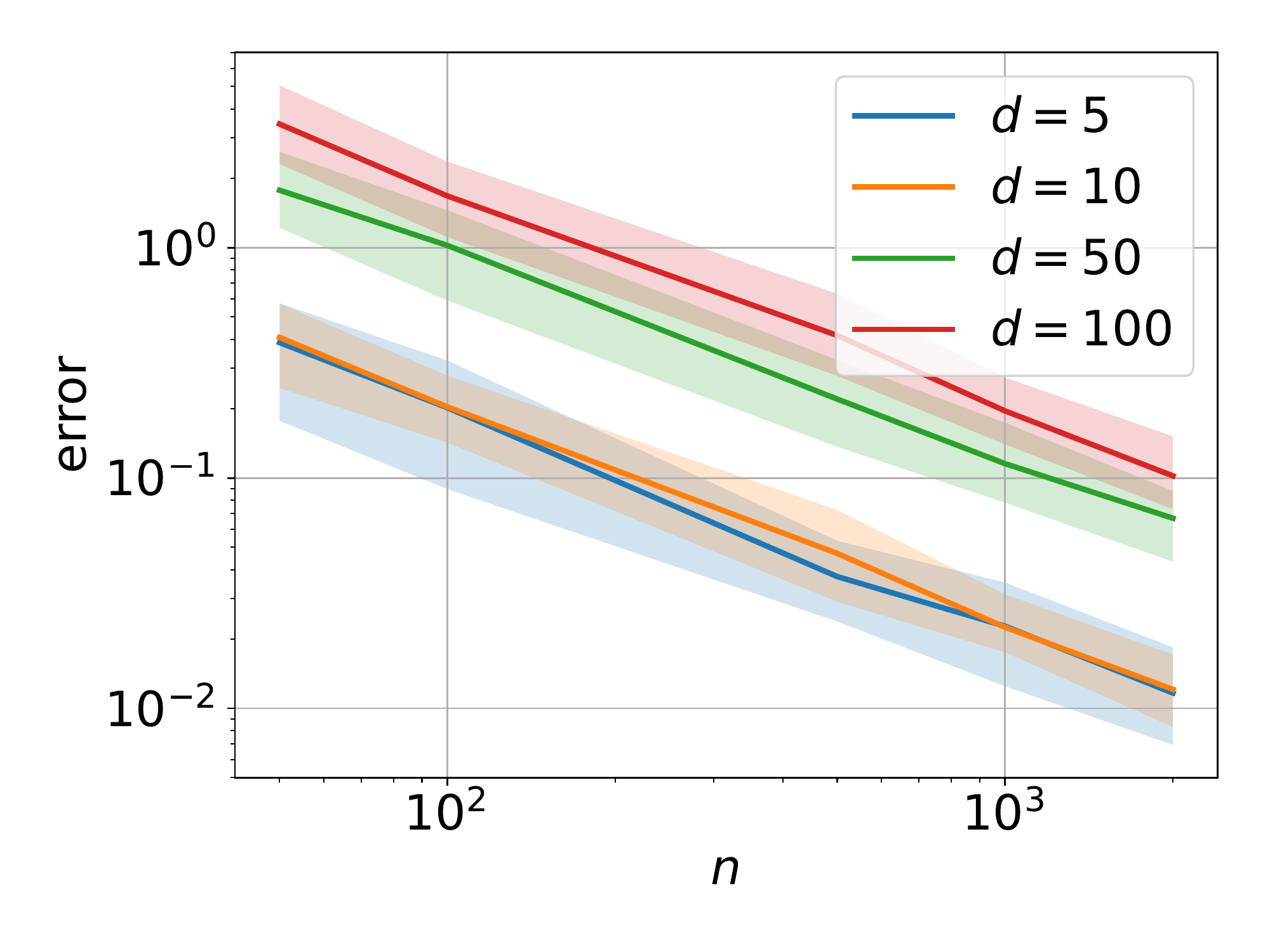}
    \includegraphics[width=.235\linewidth]{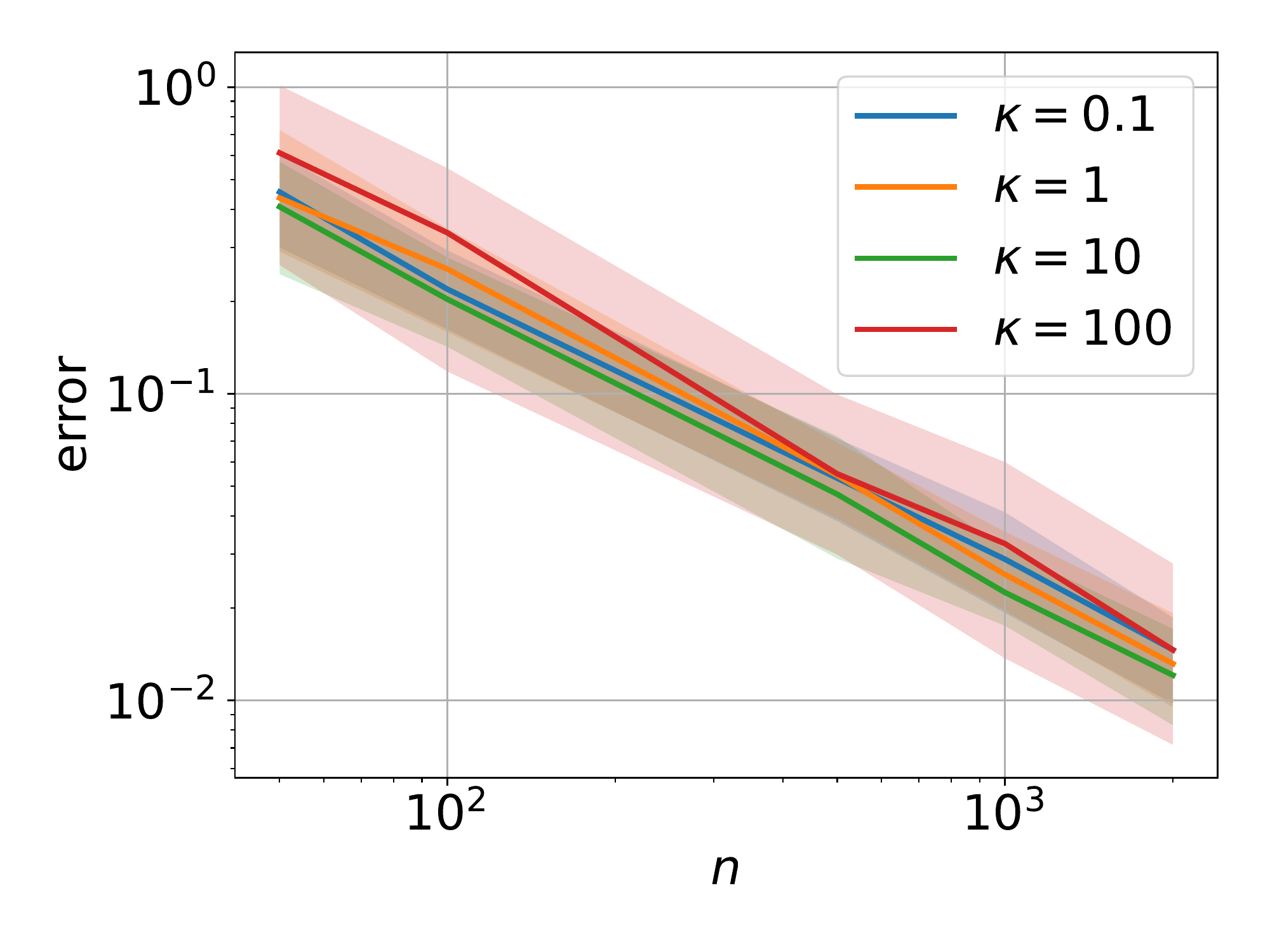}
    \label{subfig:gaussian}}
    \vspace{-3mm}
    \caption{$\swdp{\mu_n}{\nu_n}{\vMF(\mathrm{m}, \kappa)}$ vs. $n$. Results are averaged over 30 runs, on log-log scale, with 10th-90th percentiles.}
    \label{fig:illustration_bound}
    \vspace{-3mm}
\end{figure*}
\vspace{-1em}
\begin{definition} \label{def:vmf}
    The \emph{von Mises-Fisher distribution}  $\vMF(\mathrm{m}, \kappa)$ with \emph{mean direction} $\mathrm{m} \in \Sphere$ and \emph{concentration parameter} $\kappa \in \R^*_+$ is a distribution on $\Sphere$ whose density is defined for $\theta \in \Sphere$ by $\vMF(\theta ; \mathrm{m}, \kappa) = C_{d/2}(\kappa) \exp(\kappa \mathrm{m}^\top \theta)$, $C_{d/2}(\kappa) = \kappa^{d/2 - 1} / \{(2\pi)^{d/2} I_{d/2-1}(\kappa)\}$, with $I_{d/2-1}$ the \emph{modified Bessel function of the first kind at order $d/2-1$}. 
\end{definition}

Intuitively, the higher $\kappa$, the more concentrated $\vMF(\mathrm{m}, \kappa)$ is around $\mathrm{m}$. Our objective becomes finding $(\mathrm{m}^\star, \kappa^\star)$ such that $\vMF(\mathrm{m}^\star, \kappa^\star)$ maximizes \eqref{eq:rho_star} over $\mathcal{F} = \{ \vMF(\mathrm{m}, \kappa), \mathrm{m} \in \Sphere, \kappa \in \R_+^*\}$. Von Mises-Fisher distributions have been successfully deployed in several 
machine learning problems to effectively model spherical data \citep{hasnat2017mises,kumar2018mises, scott2021}. Besides, one main advantage of using vMF is that both the KL divergence between $\rho = \vMF(\mathrm{m}, \kappa)$ and $\rho_0 = \calU(\Sphere)$ and its gradient with respect to $(\mathrm{m}, \kappa)$ admit an analytical formula \citep{davidson2018}.

While the vMF parameterization is practical, as it yields an analytical objective, it may suffer from a lack of expressivity (e.g., vMF distributions are unimodal). To handle more complicated data, we also consider the parameterization proposed in \cite{nguyen2021distributional}: we solve \eqref{eq:rho_star} over $\mathcal{F} = \{ \rho = f_\sharp \calU(\Sphere), f \text{ a neural network} \}$. In that case, the KL penalty term is intractable and we approximate it with the methodology in \citep{ghimire2021reliable} -- where approximation errors of the KL estimator are given in different scenarios.

We approximate the solution of \eqref{eq:rho_star} via gradient ascent: our methodology is depicted in \Cref{alg:PACSW-algorithm}, and specialized in \Cref{alg:vmf_sw} for the vMF parameterization.

\textbf{Tuning $\lambda$.}\; In classical PAC-Bayesian theory, $\lambda$ is usually set to $n^{1/2}$ so that all terms in the bound that depend on $\lambda$ converge at the same rate to 0, as $n$ grows to $\pinf$. Nevertheless, using $\lambda = n^{\alpha}$ with $\alpha \in (0, 1)$, $\alpha \neq 1/2$ can be more useful in some specific settings. For instance, a common issue when optimizing PAC-Bayesian bounds is that the objective can be dominated by the KL term \citep{cheriefabdellatif2022}. To overcome this, one can downweight the KL term by using $\alpha > 1/2$, or more sophisticated schemes \citep{blundell2015}. On the other hand, as shown in \Cref{sec:a1}, \ref{sec:a2_subg} and \ref{subsec:bernstein}, $\varphi_{\mu, \nu, p}$ depend on parameters related to the properties of $\mu, \nu$, which cannot be easily controlled in practice. Choosing $\lambda = n^{\alpha}$ with $\alpha < 1/2$ helps attenuate their influence on the objective \citep{haddouche2021pac}.

\section{Numerical Experiments} 
\label{sec:expes}

We conduct an empirical analysis to confirm our theoretical contributions and illustrate their consequences in practice, on both synthetic and real data. More details on our experimental setup are given in \Cref{appendix:additional_exp}, and the code is available at \href{https://github.com/rubenohana/PAC-Bayesian_Sliced-Wasserstein}{https://github.com/rubenohana/PAC-Bayesian\_Sliced-Wasserstein}.

\begin{figure}[t!]
    \centering
    \subfigure[$d=5$]{
    \includegraphics[width=.47\linewidth]{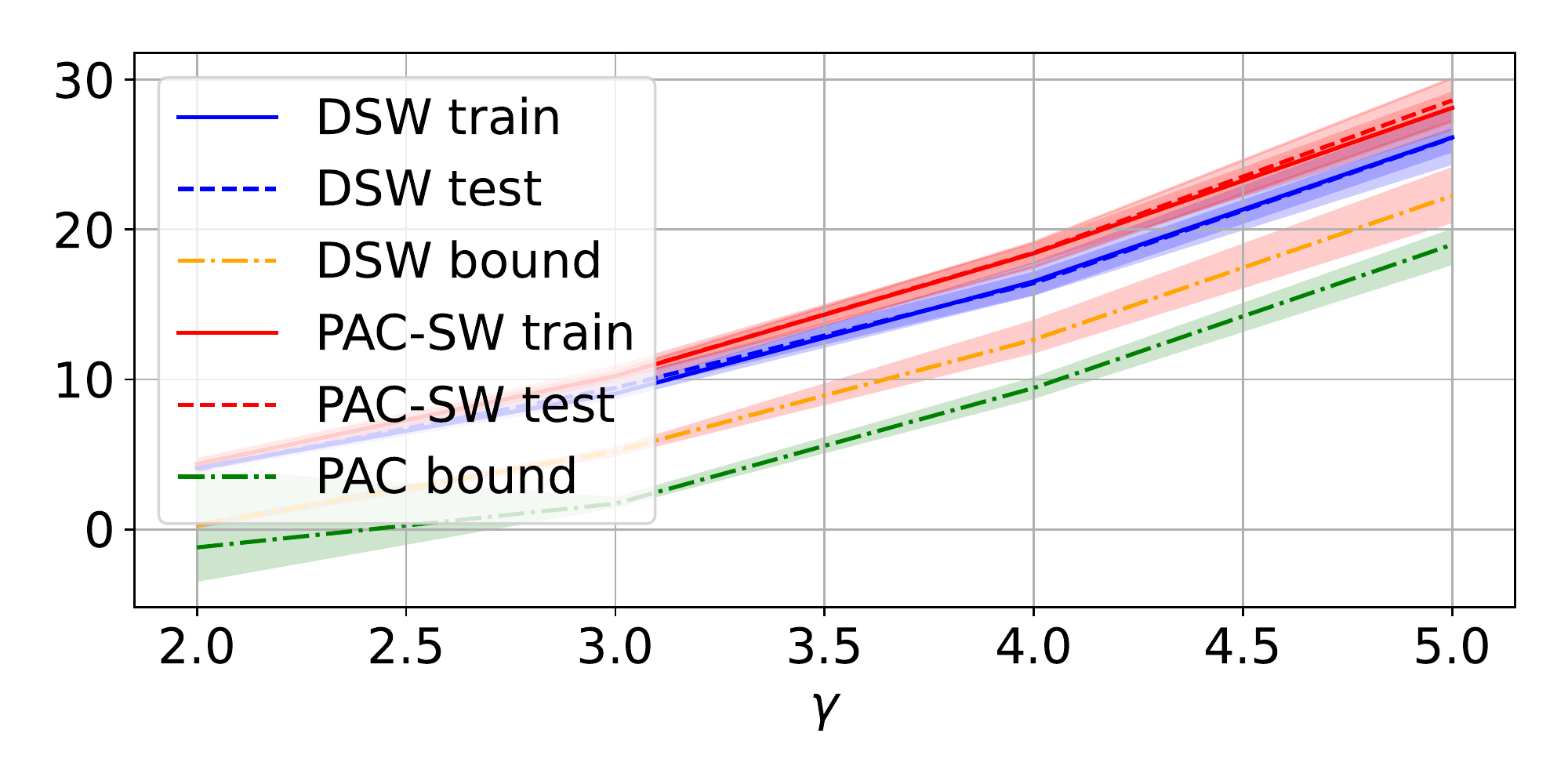}
    }
    \subfigure[$d=20$]{
    \includegraphics[width=.47\linewidth]{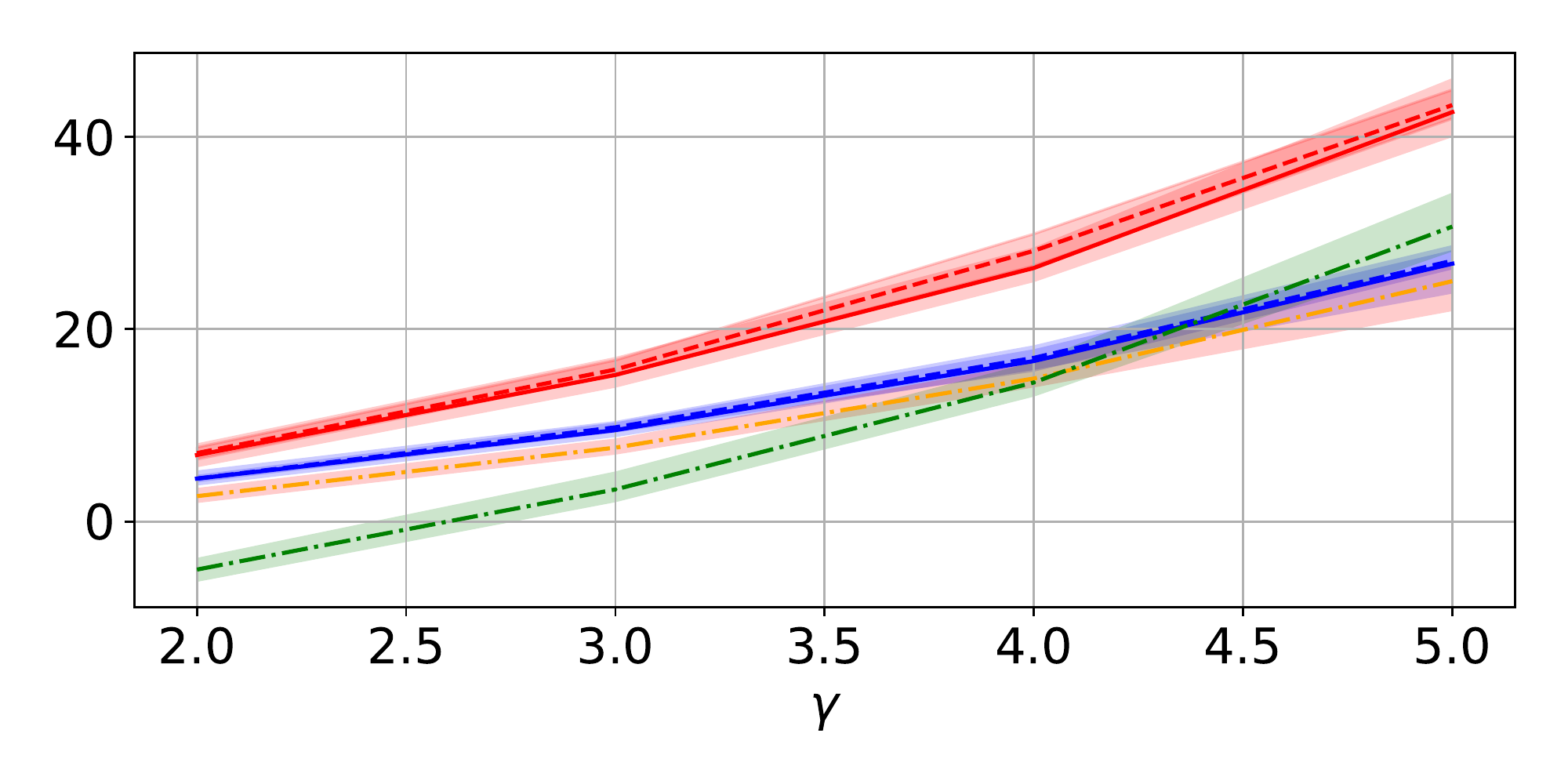}
    }
    \vspace{-3mm}
    \caption{PAC-SW and DSW between $\mu = \calN({\bf0}, \Sigma_d)$ and $\nu = \calN(\gamma{\bf1}, \Sigma_d)$. The $y$-axis shows the distances or the associated objective functions (see legend). Results are averaged over 10 runs, and shown with 10th-90th percentiles.}
    \label{fig:comparison_dsw}
    \vspace{-5mm}
\end{figure}

\textbf{Illustration of our bounds.}\; Our first set of experiments aims at empirically validating the rates of convergence in \Cref{sec:pacbound}. We sample two sets of $n$ \iid~samples from the same distribution $\mu \in \calP(\R^d)$. To illustrate our bound on both bounded and unbounded supports, we choose $\mu$ as a uniform or Gaussian distribution. We approximate $\swdp{\mu_n}{\nu_n}{\vMF(\mathrm{m}, \kappa)}$ with $\mathrm{m} \sim \calU(\Sphere)$ and $\kappa > 0$ by its Monte Carlo estimate \eqref{eq:swmc_estimate} over $1000$ projection directions. \Cref{fig:illustration_bound} plots the approximation error (which reduces to $\swdp{\mu_n}{\nu_n}{\vMF(\mathrm{m}, \kappa)}$, since the two datasets come from the same distribution) against $n$, for different $d$ and $\kappa$. We observe that the error decays to 0 as $n$ increases, and the convergence rate is slower as $d$ and $\kappa$ increase. This confirms our theoretical analysis: the higher $d$, the larger the diameter (resp., the sub-Gaussian diameter) when $\mu$ is uniform (resp., Gaussian), the larger $\varphi_{\mu, \nu, p}$ (\Cref{prop:momgenfn_compact,prop:momgenfn_subg}). Besides, the higher $\kappa$, the larger $\kld{\vMF(\mathrm{m}, \kappa)}{\calU(\Sphere)}$.

\begin{figure*}[t!]
    \centering
    \includegraphics[width = 0.95\linewidth]{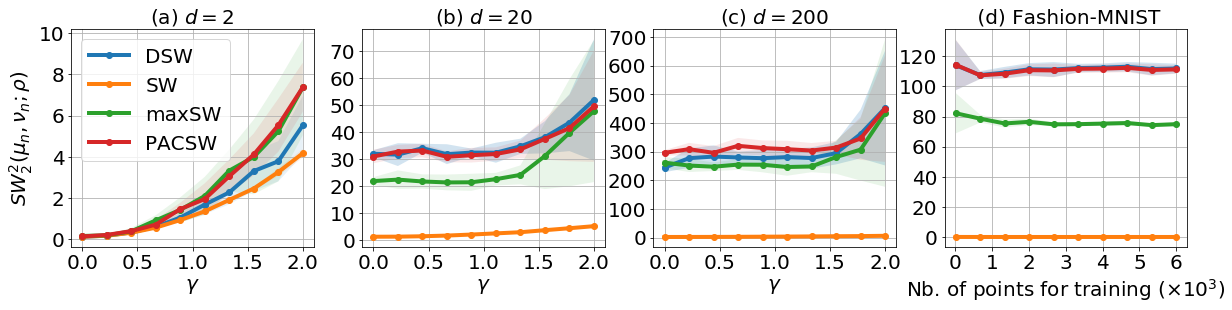}
    \vspace{-5mm}
    \caption{$\swdp{\mu_n}{\nu_n}{\rho}$ with (a-c) $\mu = \calN({\bf0}, \Sigma_d)$, $\nu = \calN(\gamma{\bf1}, \Sigma_d)$, $n = 1000$, against $\gamma$, (d) classes 4 and 5 of Fashion-MNIST, against $n$. $\rho$ is learned on the train set, and we report values on the test set.}
    \label{fig: Gaussian results diff cov}
    \vspace{-5mm}
\end{figure*}
\textbf{Generalization ability of PAC-SW.}\; Next, we study the generalization properties of PAC-SW, \ie~whether the adaptive SW computed by \Cref{alg:PACSW-algorithm} is discriminative, even on unseen data. We compare $\mu = \calN({\bf0}, \Sigma_d)$ and $\nu = \calN(\gamma {\bf1}, \Sigma_d)$, with $\gamma > 0$, $\Sigma_d \in \R^{d\times d}$ symmetric positive semi-definite set at random, and ${\bf0}$ (resp., ${\bf1}$) the vector whose components are all equal to 0 (resp., 1). The higher $\gamma$, the more dissimilar $\mu$ and $\nu$. We sample $n = 500$ samples from $\mu$ and $\nu$ and optimize $\rho^\star(\mu_n, \nu_n)$: the optimization is performed on the space of vMF distributions, using Adam \citep{kingma2014method} with its default parameters. To analyze the generalization properties of $\rho^\star(\mu_n, \nu_n)$, we sample $l = 2000$ test points from $\mu, \nu$ and compute 
$\swdp{\mu_{l}}{\nu_{l}}{\rho^\star(\mu_n, \nu_n)}$. We also compute the value of \eqref{eq:rho_star}, to evaluate the tightness of our bound. Results for different values of $d$ and $\gamma$ are reported in \Cref{fig:comparison_dsw}, and confirm the generalization ability of $\rho^\star(\mu_n, \nu_n)$.

\textbf{Comparison to existing instances of SW.}\; In our previous experiment, we also implement a variant of DSW, which consists in solving \citep[Definition 2]{nguyen2021distributional} based on our vMF parameterization.
\Cref{fig:comparison_dsw} shows that the gap between $\swdp{\mu_n}{\nu_n}{\rho^\star_{\text{DSW}}(\mu_n, \nu_n)}$ and $\swdp{\mu_m}{\nu_m}{\rho^\star_{\text{DSW}}(\mu_n, \nu_n)}$ is small, hence $\rho^\star_{\text{DSW}}(\mu_n, \nu_n)$ generalizes well on that setup. DSW bound in \Cref{fig:comparison_dsw} corresponds to the associated objective function of \citep[Definition 2]{nguyen2021distributional}.

Next, we compare the generalization properties of PAC-SW and DSW, with $\rho$ parameterized as a neural network. We also evaluate max-SW and SW (\ie, $\swdp{\cdot}{\cdot}{\calU(\Sphere}$). We compute the Monte Carlo estimate with $m = 200$ and the learning rate $\eta$ is taken as the best (\ie, yielding the higher distance) out of  $[10^{-3}, 10^{-2}, 10^{-1},1]$. Each run is averaged 10 times with standard deviations in shaded areas. On Figure \ref{fig: Gaussian results diff cov}(a-c), we measure the distance between two Gaussians, as in \Cref{fig:comparison_dsw}. We observe that PAC-SW is always amongst the most discriminative distances, and since we evaluate distances on unseen data, this implies it has better generalization properties. On \Cref{fig: Gaussian results diff cov}(d), we consider a more complicated dataset: we measure the distance between 2 highly dissimilar classes of the Fashion-MNIST dataset \citep{xiao2017fashion} (classes 4 (\textit{coats}) and 5 (\textit{sandals})) for different number of training points. PAC-SW and DSW return higher values than max-SW and SW, illustrating they are able to better discriminate, even on test data.

Note that max-SW and DSW share a common feature: they learn a new distribution $\rho(\mu_n, \nu_n)$ every time they are called on new $(\mu_n, \nu_n)$, \ie~they embed an optimization step. From here on, when we will refer to the generalization ability of max-DSW and DSW, it must be understood that a distribution $\rho^{\star}$ is learned from one sample pair $(\mu_n,\nu_n)$ according to their respective induction principle, and $\rho^{\star}$ is used on test data to measure the generalization ability.

\textbf{Generalization for generative modeling.}\; In our previous experiments, we observed that DSW can generalize as well as PAC-SW. This encourages us to further explore the advantages of a high generalization ability on a more complicated setup. 
We consider a generative modeling task on MNIST data \citep{deng2012mnist}, and we train a deep neural network that uses DSW as a loss, in the flavor of \citep{deshpande2018generative,nguyen2021distributional}. Usually, the distribution $\rho$ is learned at each iteration, when the user receive new data. We conjecture that if the learned distribution generalizes well to unseen datasets, then gradients obtained from the distance between minibatches would still provide sufficient information to learn the generative model. As a consequence, we evaluate the robustness and generalization ability of the learned distribution using DSW updated only every $10$ or $50$ minibatches  (denoted by $-10$ or $-50$ resp.). To train the model, we followed the same approach (architecture and optimizer) as the one described in \cite{nguyen2021distributional}. For each minibatch of size $512$, the distribution $\rho$ is learned by optimizing $100$ projections over $100$ iterations and the generative model is trained over $400$ epochs. We also report results of a generative model trained with max-SW.

\begin{figure}
    \centering
\includegraphics[width=0.8\linewidth]{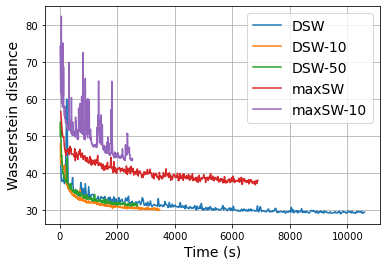}    
\vspace{-3mm}
\caption{Evolution of the Wasserstein distance between a set of generated MNIST digits and the true MNIST test set with respect to training time.}\label{fig:genmodevolution}
\vspace{-6mm}
\end{figure}

Figure \ref{fig:genmodevolution} shows the evolution of the Wasserstein distance (WD) between generated data and the test set with respect to training time (measured after each epoch), for each distance and different update rate of the distribution $\rho$.
We can observe that
classical DSW yields a WD of $29$ after $\sim10^4$s. When learning $\rho$ every $10$ minibatch (DSW-10), we achieve similar a WD value with half the running time. When further reducing the
frequency update of $\rho$ (DSW-50), we converge faster but with a loss in quality of generation (WD $\sim32$). While using max-SW as a loss yields a reasonable performance, computing $\rho^\star_{\text{maxSW}}$ every $10$ minibatches leads to a very unstable learning and worst performances. Results for the PAC-SW loss and examples of generated digits can be found in \Cref{secapp: generated MNIST}.

\section{Conclusion}

We introduced a specific notion of generalization for adaptive SW, related to discriminative power, and leveraged the PAC-Bayesian framework to derive generalization bounds. We then developed a principled methodology to learn $\rho$ from the observed data so as $\swdp{\cdot}{\cdot}{\rho}$ is discriminative with high probability, thus, generalizes well. Our work, which presents the first connection between PAC-Bayes and SW, paves the way to interesting research directions. First, we will study possible refinements of our bounds, using other PAC-Bayes bounds than Catoni's. Then, we plan to further analyze why DSW generalizes well in our experiments, e.g. by investigating a potential connection between the optimization problem in \citep{nguyen2021distributional} and ours. Finally, we would like to reduce the computational complexity of PAC-SW when $\rho$ is parameterized as a neural network, since it suffers from slow execution times mainly because of the approximation of the KL term with \citep{ghimire2021reliable}.

\bibliographystyle{ICML_template/icml2023}
{\small
\bibliography{main_bib}

\begin{thebibliography}{69}
\providecommand{\natexlab}[1]{#1}
\providecommand{\url}[1]{\texttt{#1}}
\expandafter\ifx\csname urlstyle\endcsname\relax
  \providecommand{\doi}[1]{doi: #1}\else
  \providecommand{\doi}{doi: \begingroup \urlstyle{rm}\Url}\fi

\bibitem[Alquier(2021)]{alquier2021userfriendly}
Alquier, P.
\newblock User-friendly introduction to {PAC}-{B}ayes bounds, 2021.

\bibitem[Alquier \& Guedj(2018)Alquier and Guedj]{alquier2018simpler}
Alquier, P. and Guedj, B.
\newblock Simpler pac-bayesian bounds for hostile data.
\newblock \emph{Machine Learning}, 107\penalty0 (5):\penalty0 887--902, 2018.

\bibitem[Ambroladze et~al.(2007)Ambroladze, Parrado-hern\'{a}ndez, and
  Shawe-taylor]{Ambroladze06Tighter}
Ambroladze, A., Parrado-hern\'{a}ndez, E., and Shawe-taylor, J.
\newblock Tighter pac-bayes bounds.
\newblock In Sch\"{o}lkopf, B., Platt, J., and Hoffman, T. (eds.),
  \emph{Advances in Neural Information Processing Systems}, volume~19. MIT
  Press, 2007.

\bibitem[Arjovsky et~al.(2017)Arjovsky, Chintala, and
  Bottou]{arjovsky2017wasserstein}
Arjovsky, M., Chintala, S., and Bottou, L.
\newblock Wasserstein generative adversarial networks.
\newblock In \emph{International conference on machine learning}, pp.\
  214--223. PMLR, 2017.

\bibitem[Bayraktar \& Guo(2021)Bayraktar and Guo]{bayraktar2019strong}
Bayraktar, E. and Guo, G.
\newblock {Strong equivalence between metrics of Wasserstein type}.
\newblock \emph{Electronic Communications in Probability}, 26\penalty0
  (none):\penalty0 1 -- 13, 2021.
\newblock \doi{10.1214/21-ECP383}.

\bibitem[Blundell et~al.(2015)Blundell, Cornebise, Kavukcuoglu, and
  Wierstra]{blundell2015}
Blundell, C., Cornebise, J., Kavukcuoglu, K., and Wierstra, D.
\newblock Weight uncertainty in neural networks.
\newblock In \emph{Proceedings of the 32nd International Conference on
  International Conference on Machine Learning - Volume 37}, ICML'15, pp.\
  1613–1622. JMLR.org, 2015.

\bibitem[Bonet et~al.(2021)Bonet, Courty, Septier, and
  Drumetz]{bonet2021sliced}
Bonet, C., Courty, N., Septier, F., and Drumetz, L.
\newblock Sliced-wasserstein gradient flows.
\newblock \emph{arXiv preprint arXiv:2110.10972}, 2021.

\bibitem[Bonneel et~al.(2015)Bonneel, Rabin, Peyr{\'e}, and
  Pfister]{bonneel2015sliced}
Bonneel, N., Rabin, J., Peyr{\'e}, G., and Pfister, H.
\newblock Sliced and radon wasserstein barycenters of measures.
\newblock \emph{Journal of Mathematical Imaging and Vision}, 51\penalty0
  (1):\penalty0 22--45, 2015.

\bibitem[Bonnotte(2013)]{Bonnotte2013}
Bonnotte, N.
\newblock \emph{Unidimensional and Evolution Methods for Optimal
  Transportation}.
\newblock PhD thesis, Paris 11, 2013.

\bibitem[Bousquet et~al.(2017)Bousquet, Gelly, Tolstikhin, Simon-Gabriel, and
  Schoelkopf]{bousquet2017optimal}
Bousquet, O., Gelly, S., Tolstikhin, I., Simon-Gabriel, C.-J., and Schoelkopf,
  B.
\newblock From optimal transport to generative modeling: the vegan cookbook.
\newblock \emph{arXiv preprint arXiv:1705.07642}, 2017.

\bibitem[Carriere et~al.(2017)Carriere, Cuturi, and Oudot]{carriere2017sliced}
Carriere, M., Cuturi, M., and Oudot, S.
\newblock Sliced wasserstein kernel for persistence diagrams.
\newblock In \emph{International conference on machine learning}, pp.\
  664--673. PMLR, 2017.

\bibitem[Catoni(2003)]{catoni03pacbayesian}
Catoni, O.
\newblock A {PAC}-{B}ayesian approach to adaptive classification.
\newblock preprint LPMA 840, 2003.

\bibitem[Catoni(2007)]{catoni07pacbayesian}
Catoni, O.
\newblock \emph{Pac-Bayesian Supervised Classification: The Thermodynamics of
  Statistical Learning}, volume~56.
\newblock Institute of Mathematical Statistics, 2007.

\bibitem[Ch\'erief-Abdellatif et~al.(2022)Ch\'erief-Abdellatif, Shi, Doucet,
  and Guedj]{cheriefabdellatif2022}
Ch\'erief-Abdellatif, B.-E., Shi, Y., Doucet, A., and Guedj, B.
\newblock On pac-bayesian reconstruction guarantees for vaes.
\newblock In Camps-Valls, G., Ruiz, F. J.~R., and Valera, I. (eds.),
  \emph{Proceedings of The 25th International Conference on Artificial
  Intelligence and Statistics}, volume 151 of \emph{Proceedings of Machine
  Learning Research}, pp.\  3066--3079. PMLR, 28--30 Mar 2022.

\bibitem[Courty et~al.(2016)Courty, Flamary, Tuia, and
  Rakotomamonjy]{courty2016optimal}
Courty, N., Flamary, R., Tuia, D., and Rakotomamonjy, A.
\newblock Optimal transport for domain adaptation.
\newblock \emph{IEEE transactions on pattern analysis and machine
  intelligence}, 39\penalty0 (9):\penalty0 1853--1865, 2016.

\bibitem[Cuturi(2013)]{cuturi2013sinkhorn}
Cuturi, M.
\newblock Sinkhorn distances: Lightspeed computation of optimal transport.
\newblock \emph{Advances in neural information processing systems},
  26:\penalty0 2292--2300, 2013.

\bibitem[Cuturi \& Peyr{\'e}(2016)Cuturi and Peyr{\'e}]{cuturi2016smoothed}
Cuturi, M. and Peyr{\'e}, G.
\newblock A smoothed dual approach for variational wasserstein problems.
\newblock \emph{SIAM Journal on Imaging Sciences}, 9\penalty0 (1):\penalty0
  320--343, 2016.

\bibitem[Davidson et~al.(2018)Davidson, Falorsi, De~Cao, Kipf, and
  Tomczak]{davidson2018}
Davidson, T.~R., Falorsi, L., De~Cao, N., Kipf, T., and Tomczak, J.~M.
\newblock Hyperspherical variational auto-encoders.
\newblock \emph{34th Conference on Uncertainty in Artificial Intelligence
  (UAI-18)}, 2018.

\bibitem[Deng(2012)]{deng2012mnist}
Deng, L.
\newblock The mnist database of handwritten digit images for machine learning
  research.
\newblock \emph{IEEE Signal Processing Magazine}, 29\penalty0 (6):\penalty0
  141--142, 2012.

\bibitem[Deshpande et~al.(2018)Deshpande, Zhang, and
  Schwing]{deshpande2018generative}
Deshpande, I., Zhang, Z., and Schwing, A.~G.
\newblock Generative modeling using the sliced wasserstein distance.
\newblock In \emph{Proceedings of the IEEE conference on computer vision and
  pattern recognition}, pp.\  3483--3491, 2018.

\bibitem[Deshpande et~al.(2019)Deshpande, Hu, Sun, Pyrros, Siddiqui, Koyejo,
  Zhao, Forsyth, and Schwing]{deshpande2019max}
Deshpande, I., Hu, Y.-T., Sun, R., Pyrros, A., Siddiqui, N., Koyejo, S., Zhao,
  Z., Forsyth, D., and Schwing, A.~G.
\newblock Max-sliced wasserstein distance and its use for gans.
\newblock In \emph{Proceedings of the IEEE/CVF Conference on Computer Vision
  and Pattern Recognition}, pp.\  10648--10656, 2019.

\bibitem[Donsker \& Varadhan(1975)Donsker and Varadhan]{donsker1975variational}
Donsker, M.~D. and Varadhan, S.~S.
\newblock Asymptotic evaluation of certain markov process expectations for
  large time, i.
\newblock \emph{Communications on Pure and Applied Mathematics}, 28\penalty0
  (1):\penalty0 1--47, 1975.

\bibitem[Dziugaite \& Roy(2017)Dziugaite and Roy]{dziugaite2017}
Dziugaite, G.~K. and Roy, D.~M.
\newblock Computing nonvacuous generalization bounds for deep (stochastic)
  neural networks with many more parameters than training data.
\newblock In Elidan, G., Kersting, K., and Ihler, A.~T. (eds.),
  \emph{Proceedings of the Thirty-Third Conference on Uncertainty in Artificial
  Intelligence, {UAI} 2017, Sydney, Australia, August 11-15, 2017}. {AUAI}
  Press, 2017.

\bibitem[Embrechts et~al.(2013)Embrechts, Kl{\"u}ppelberg, and
  Mikosch]{embrechts2013modelling}
Embrechts, P., Kl{\"u}ppelberg, C., and Mikosch, T.
\newblock \emph{Modelling extremal events: for insurance and finance},
  volume~33.
\newblock Springer Science \& Business Media, 2013.

\bibitem[Fournier \& Guillin(2015)Fournier and Guillin]{fournier2015rate}
Fournier, N. and Guillin, A.
\newblock On the rate of convergence in wasserstein distance of the empirical
  measure.
\newblock \emph{Probability Theory and Related Fields}, 162\penalty0
  (3):\penalty0 707--738, 2015.

\bibitem[Frogner et~al.(2015)Frogner, Zhang, Mobahi, Araya-Polo, and
  Poggio]{frogner2015learning}
Frogner, C., Zhang, C., Mobahi, H., Araya-Polo, M., and Poggio, T.
\newblock Learning with a wasserstein loss.
\newblock \emph{arXiv preprint arXiv:1506.05439}, 2015.

\bibitem[Germain et~al.(2009)Germain, Lacasse, Laviolette, and
  Marchand]{Germain09pacbayesian}
Germain, P., Lacasse, A., Laviolette, F., and Marchand, M.
\newblock Pac-bayesian learning of linear classifiers.
\newblock In \emph{Proc. of the 26th International Conference on Machine
  Learning (ICML)}, pp.\  353--360, 2009.

\bibitem[Ghimire et~al.(2021)Ghimire, Masoomi, and Dy]{ghimire2021reliable}
Ghimire, S., Masoomi, A., and Dy, J.
\newblock Reliable estimation of kl divergence using a discriminator in
  reproducing kernel hilbert space.
\newblock \emph{Advances in Neural Information Processing Systems}, 34, 2021.

\bibitem[Haddouche et~al.(2021)Haddouche, Guedj, Rivasplata, and
  Shawe-Taylor]{haddouche2021pac}
Haddouche, M., Guedj, B., Rivasplata, O., and Shawe-Taylor, J.
\newblock Pac-bayes unleashed: Generalisation bounds with unbounded losses.
\newblock \emph{Entropy}, 23\penalty0 (10), 2021.
\newblock ISSN 1099-4300.
\newblock \doi{10.3390/e23101330}.

\bibitem[Hasnat et~al.(2017)Hasnat, Bohn{\'e}, Milgram, Gentric, Chen,
  et~al.]{hasnat2017mises}
Hasnat, M., Bohn{\'e}, J., Milgram, J., Gentric, S., Chen, L., et~al.
\newblock von mises-fisher mixture model-based deep learning: Application to
  face verification.
\newblock \emph{arXiv preprint arXiv:1706.04264}, 2017.

\bibitem[Higgs \& Shawe-Taylor(2010)Higgs and Shawe-Taylor]{higgs2010}
Higgs, M. and Shawe-Taylor, J.
\newblock A pac-bayes bound for tailored density estimation.
\newblock In \emph{Proceedings of the 21st International Conference on
  Algorithmic Learning Theory}, ALT'10, pp.\  148–162, Berlin, Heidelberg,
  2010. Springer-Verlag.
\newblock ISBN 3642161073.

\bibitem[Kingma \& Ba(2015)Kingma and Ba]{kingma2014method}
Kingma, D.~P. and Ba, J.
\newblock {Adam: A Method for Stochastic Optimization}.
\newblock In Bengio, Y. and LeCun, Y. (eds.), \emph{3rd International
  Conference on Learning Representations, {ICLR} 2015, San Diego, CA, USA, May
  7-9, 2015, Conference Track Proceedings}, 2015.

\bibitem[Kolouri et~al.(2016)Kolouri, Zou, and Rohde]{kolouri2016sliced}
Kolouri, S., Zou, Y., and Rohde, G.~K.
\newblock Sliced wasserstein kernels for probability distributions.
\newblock In \emph{Proceedings of the IEEE Conference on Computer Vision and
  Pattern Recognition}, pp.\  5258--5267, 2016.

\bibitem[Kolouri et~al.(2017)Kolouri, Park, Thorpe, Slepcev, and
  Rohde]{kolouri2017optimal}
Kolouri, S., Park, S.~R., Thorpe, M., Slepcev, D., and Rohde, G.~K.
\newblock Optimal mass transport: Signal processing and machine-learning
  applications.
\newblock \emph{IEEE signal processing magazine}, 34\penalty0 (4):\penalty0
  43--59, 2017.

\bibitem[Kolouri et~al.(2018)Kolouri, Rohde, and Hoffmann]{kolouri2018sliced}
Kolouri, S., Rohde, G.~K., and Hoffmann, H.
\newblock Sliced wasserstein distance for learning gaussian mixture models.
\newblock In \emph{Proceedings of the IEEE Conference on Computer Vision and
  Pattern Recognition}, pp.\  3427--3436, 2018.

\bibitem[Kolouri et~al.(2019{\natexlab{a}})Kolouri, Nadjahi, Simsekli, Badeau,
  and Rohde]{kolouri2019generalized}
Kolouri, S., Nadjahi, K., Simsekli, U., Badeau, R., and Rohde, G.~K.
\newblock Generalized sliced wasserstein distances.
\newblock \emph{arXiv preprint arXiv:1902.00434}, 2019{\natexlab{a}}.

\bibitem[Kolouri et~al.(2019{\natexlab{b}})Kolouri, Pope, Martin, and
  Rohde]{kolouri2019sliced}
Kolouri, S., Pope, P.~E., Martin, C.~E., and Rohde, G.~K.
\newblock Sliced wasserstein auto-encoders.
\newblock In \emph{International Conference on Learning Representations},
  2019{\natexlab{b}}.

\bibitem[Kolouri et~al.(2020)Kolouri, Nadjahi, Simsekli, and
  Shahrampour]{kolouri2020generalized}
Kolouri, S., Nadjahi, K., Simsekli, U., and Shahrampour, S.
\newblock Generalized sliced distances for probability distributions.
\newblock \emph{arXiv preprint arXiv:2002.12537}, 2020.

\bibitem[Kontorovich(2014)]{kontorovicha14}
Kontorovich, A.
\newblock Concentration in unbounded metric spaces and algorithmic stability.
\newblock In Xing, E.~P. and Jebara, T. (eds.), \emph{Proceedings of the 31st
  International Conference on Machine Learning}, volume~32 of \emph{Proceedings
  of Machine Learning Research}, pp.\  28--36, Bejing, China, 22--24 Jun 2014.
  PMLR.

\bibitem[Kumar \& Tsvetkov(2018)Kumar and Tsvetkov]{kumar2018mises}
Kumar, S. and Tsvetkov, Y.
\newblock Von mises-fisher loss for training sequence to sequence models with
  continuous outputs.
\newblock \emph{arXiv preprint arXiv:1812.04616}, 2018.

\bibitem[Laviolette et~al.(2006)Laviolette, Marchand, and
  Shah]{laviolette06pacbayes}
Laviolette, F., Marchand, M., and Shah, M.
\newblock A pac-bayes approach to the set covering machine.
\newblock In Weiss, Y., Sch\"{o}lkopf, B., and Platt, J. (eds.), \emph{Advances
  in Neural Information Processing Systems}, volume~18. MIT Press, 2006.

\bibitem[Lei(2020)]{lei2020}
Lei, J.
\newblock {Convergence and concentration of empirical measures under
  Wasserstein distance in unbounded functional spaces}.
\newblock \emph{Bernoulli}, 26\penalty0 (1):\penalty0 767 -- 798, 2020.
\newblock \doi{10.3150/19-BEJ1151}.

\bibitem[Lin et~al.(2021)Lin, Zheng, Chen, Cuturi, and Jordan]{lin2021}
Lin, T., Zheng, Z., Chen, E.~Y., Cuturi, M., and Jordan, M.~I.
\newblock On projection robust optimal transport: Sample complexity and model
  misspecification.
\newblock In \emph{AISTATS}, pp.\  262--270, 2021.

\bibitem[Liutkus et~al.(2019)Liutkus, Simsekli, Majewski, Durmus, and
  St{\"o}ter]{liutkus2019sliced}
Liutkus, A., Simsekli, U., Majewski, S., Durmus, A., and St{\"o}ter, F.-R.
\newblock Sliced-wasserstein flows: Nonparametric generative modeling via
  optimal transport and diffusions.
\newblock In \emph{International Conference on Machine Learning}, pp.\
  4104--4113. PMLR, 2019.

\bibitem[Manole et~al.(2022)Manole, Balakrishnan, and
  Wasserman]{manole2020minimax}
Manole, T., Balakrishnan, S., and Wasserman, L.
\newblock {Minimax confidence intervals for the Sliced Wasserstein distance}.
\newblock \emph{Electronic Journal of Statistics}, 16\penalty0 (1):\penalty0
  2252 -- 2345, 2022.
\newblock \doi{10.1214/22-EJS2001}.

\bibitem[McAllester(1999)]{mcallester1999some}
McAllester, D.~A.
\newblock Some {PAC}-{B}ayesian theorems.
\newblock \emph{Machine Learning}, 37\penalty0 (3):\penalty0 355--363, 1999.

\bibitem[McDiarmid(1989)]{mcdiarmid1989method}
McDiarmid, C.
\newblock On the method of bounded differences.
\newblock \emph{Surveys in combinatorics}, 141\penalty0 (1):\penalty0 148--188,
  1989.

\bibitem[Montavon et~al.(2016)Montavon, M{\"u}ller, and
  Cuturi]{montavon2016wasserstein}
Montavon, G., M{\"u}ller, K.-R., and Cuturi, M.
\newblock Wasserstein training of restricted boltzmann machines.
\newblock \emph{Advances in Neural Information Processing Systems},
  29:\penalty0 3718--3726, 2016.

\bibitem[Nadjahi et~al.(2019)Nadjahi, Durmus, {\c{S}}im{\c{s}}ekli, and
  Badeau]{nadjahi2019asymptotic}
Nadjahi, K., Durmus, A., {\c{S}}im{\c{s}}ekli, U., and Badeau, R.
\newblock Asymptotic guarantees for learning generative models with the
  sliced-wasserstein distance.
\newblock \emph{arXiv preprint arXiv:1906.04516}, 2019.

\bibitem[Nadjahi et~al.(2020{\natexlab{a}})Nadjahi, De~Bortoli, Durmus, Badeau,
  and {\c{S}}im{\c{s}}ekli]{nadjahi2020approximate}
Nadjahi, K., De~Bortoli, V., Durmus, A., Badeau, R., and {\c{S}}im{\c{s}}ekli,
  U.
\newblock Approximate bayesian computation with the sliced-wasserstein
  distance.
\newblock In \emph{ICASSP 2020-2020 IEEE International Conference on Acoustics,
  Speech and Signal Processing (ICASSP)}, pp.\  5470--5474. IEEE,
  2020{\natexlab{a}}.

\bibitem[Nadjahi et~al.(2020{\natexlab{b}})Nadjahi, Durmus, Chizat, Kolouri,
  Shahrampour, and {\c{S}}im{\c{s}}ekli]{nadjahi2020statistical}
Nadjahi, K., Durmus, A., Chizat, L., Kolouri, S., Shahrampour, S., and
  {\c{S}}im{\c{s}}ekli, U.
\newblock Statistical and topological properties of sliced probability
  divergences.
\newblock \emph{arXiv preprint arXiv:2003.05783}, 2020{\natexlab{b}}.

\bibitem[Nadjahi et~al.(2021)Nadjahi, Durmus, Jacob, Badeau, and
  Simsekli]{nadjahi2021fast}
Nadjahi, K., Durmus, A., Jacob, P., Badeau, R., and Simsekli, U.
\newblock Fast approximation of the sliced-wasserstein distance using
  concentration of random projections.
\newblock In Beygelzimer, A., Dauphin, Y., Liang, P., and Vaughan, J.~W.
  (eds.), \emph{Advances in Neural Information Processing Systems}, 2021.

\bibitem[Nguyen et~al.(2021)Nguyen, Ho, Pham, and
  Bui]{nguyen2021distributional}
Nguyen, K., Ho, N., Pham, T., and Bui, H.
\newblock Distributional sliced-wasserstein and applications to generative
  modeling.
\newblock In \emph{International Conference on Learning Representations}, 2021.

\bibitem[Niles-Weed \& Rigollet(2022)Niles-Weed and Rigollet]{weed2022}
Niles-Weed, J. and Rigollet, P.
\newblock {Estimation of Wasserstein distances in the Spiked Transport Model}.
\newblock \emph{Bernoulli}, 28\penalty0 (4):\penalty0 2663 -- 2688, 2022.
\newblock \doi{10.3150/21-BEJ1433}.

\bibitem[Oberman \& Ruan(2015)Oberman and Ruan]{oberman2015efficient}
Oberman, A.~M. and Ruan, Y.
\newblock An efficient linear programming method for optimal transportation.
\newblock \emph{arXiv preprint arXiv:1509.03668}, 2015.

\bibitem[Peyr{\'e} \& Cuturi(2019)Peyr{\'e} and Cuturi]{peyre2019computational}
Peyr{\'e}, G. and Cuturi, M.
\newblock Computational optimal transport: With applications to data science.
\newblock \emph{Foundations and Trends® in Machine Learning}, 11\penalty0
  (5-6):\penalty0 355--607, 2019.
\newblock ISSN 1935-8237.
\newblock \doi{10.1561/2200000073}.

\bibitem[Rabin et~al.(2012)Rabin, Peyr{\'e}, Delon, and
  Bernot]{rabin2012wasserstein}
Rabin, J., Peyr{\'e}, G., Delon, J., and Bernot, M.
\newblock Wasserstein {B}arycenter and {I}ts {A}pplication to {T}exture
  {M}ixing.
\newblock In Bruckstein, A.~M., ter Haar~Romeny, B.~M., Bronstein, A.~M., and
  Bronstein, M.~M. (eds.), \emph{Scale Space and Variational Methods in
  Computer Vision}, pp.\  435--446, Berlin, Heidelberg, 2012. Springer Berlin
  Heidelberg.
\newblock ISBN 978-3-642-24785-9.

\bibitem[Rakotomamonjy \& Ralaivola(2021)Rakotomamonjy and
  Ralaivola]{rakotomamonjy2021differentially}
Rakotomamonjy, A. and Ralaivola, L.
\newblock Differentially private sliced wasserstein distance.
\newblock In \emph{International Conference on Machine Learning}, pp.\
  8810--8820. PMLR, 2021.

\bibitem[Ralaivola et~al.(2010)Ralaivola, Szafranski, and
  Stempfel]{ralaivola2010}
Ralaivola, L., Szafranski, M., and Stempfel, G.
\newblock {Chromatic PAC-Bayes Bounds for Non-IID Data: Applications to Ranking
  and Stationary $\beta$-Mixing Processes}.
\newblock \emph{Journal of Machine Learning Research}, 11\penalty0
  (65):\penalty0 1927--1956, 2010.

\bibitem[Rivasplata(2012)]{Rivasplata2012SubgaussianRV}
Rivasplata, O.
\newblock Subgaussian random variables : An expository note.
\newblock 2012.

\bibitem[Schmitz et~al.(2018)Schmitz, Heitz, Bonneel, Ngole, Coeurjolly,
  Cuturi, Peyr{\'e}, and Starck]{schmitz2018wasserstein}
Schmitz, M.~A., Heitz, M., Bonneel, N., Ngole, F., Coeurjolly, D., Cuturi, M.,
  Peyr{\'e}, G., and Starck, J.-L.
\newblock Wasserstein dictionary learning: Optimal transport-based unsupervised
  nonlinear dictionary learning.
\newblock \emph{SIAM Journal on Imaging Sciences}, 11\penalty0 (1):\penalty0
  643--678, 2018.

\bibitem[Schmitzer(2016)]{schmitzer2016sparse}
Schmitzer, B.
\newblock A sparse multiscale algorithm for dense optimal transport.
\newblock \emph{Journal of Mathematical Imaging and Vision}, 56\penalty0
  (2):\penalty0 238--259, 2016.

\bibitem[Scott et~al.(2021)Scott, Gallagher, and Mozer]{scott2021}
Scott, T.~R., Gallagher, A.~C., and Mozer, M.~C.
\newblock von mises–fisher loss: An exploration of embedding geometries for
  supervised learning.
\newblock In \emph{2021 IEEE/CVF International Conference on Computer Vision
  (ICCV)}, pp.\  10592--10602, Los Alamitos, CA, USA, oct 2021. IEEE Computer
  Society.
\newblock \doi{10.1109/ICCV48922.2021.01044}.

\bibitem[Solomon et~al.(2014)Solomon, Rustamov, Guibas, and
  Butscher]{solomon2014wasserstein}
Solomon, J., Rustamov, R., Guibas, L., and Butscher, A.
\newblock Wasserstein propagation for semi-supervised learning.
\newblock In \emph{International Conference on Machine Learning}, pp.\
  306--314. PMLR, 2014.

\bibitem[Tolstikhin et~al.(2017)Tolstikhin, Bousquet, Gelly, and
  Schoelkopf]{tolstikhin2017wasserstein}
Tolstikhin, I., Bousquet, O., Gelly, S., and Schoelkopf, B.
\newblock Wasserstein auto-encoders.
\newblock \emph{arXiv preprint arXiv:1711.01558}, 2017.

\bibitem[Villani(2009)]{villani2009optimal}
Villani, C.
\newblock \emph{Optimal transport: old and new}, volume 338.
\newblock Springer, 2009.

\bibitem[Weed \& Bach(2019)Weed and Bach]{weedbach2019}
Weed, J. and Bach, F.
\newblock Sharp asymptotic and finite-sample rates of convergence of empirical
  measures in {W}asserstein distance.
\newblock \emph{Bernoulli}, 25\penalty0 (4 A):\penalty0 2620--2648, 2019.
\newblock ISSN 1350-7265.
\newblock \doi{10.3150/18-BEJ1065}.

\bibitem[Xiao et~al.(2017)Xiao, Rasul, and Vollgraf]{xiao2017fashion}
Xiao, H., Rasul, K., and Vollgraf, R.
\newblock Fashion-mnist: a novel image dataset for benchmarking machine
  learning algorithms.
\newblock \emph{arXiv preprint arXiv:1708.07747}, 2017.

\bibitem[Zantedeschi et~al.(2021)Zantedeschi, Viallard, Morvant, Emonet,
  Habrard, Germain, and Guedj]{zantedeschi21stochastic}
Zantedeschi, V., Viallard, P., Morvant, E., Emonet, R., Habrard, A., Germain,
  P., and Guedj, B.
\newblock {Learning Stochastic Majority Votes by Minimizing a PAC-Bayes
  Generalization Bound}.
\newblock In \emph{{NeurIPS}}, Online, France, 2021.

\end{thebibliography}
}
\appendix

\setcounter{equation}{0}
\setcounter{figure}{0}
\setcounter{table}{0}
\renewcommand{\thesection}{A\arabic{section}}
\renewcommand{\theequation}{A\arabic{equation}}
\renewcommand{\thefigure}{A\arabic{figure}}
\renewcommand{\thedefinition}{A\arabic{definition}}
\renewcommand{\thetheorem}{A\arabic{theorem}}
\renewcommand{\thelemma}{A\arabic{lemma}}
\renewcommand{\theproposition}{A\arabic{proposition}}
\renewcommand{\thecorollary}{A\arabic{corollary}}
\renewcommand{\thealgorithm}{A\arabic{algorithm}} 



%

%

\onecolumn


\section{Preliminaries}

\subsection{Metric Properties of Sliced-Wasserstein Distances} \label{appendix:metric_sw}

Previous work have shown that for specific instances of $\rho \in \calP(\Sphere)$, $\swd{\cdot}{\cdot}{\rho} : \calP_p(\R^d) \times \calP_p(\R^d) \to \R_+$ is a metric, as it satisfies all metric axioms (positivity, symmetry, triangle inequality, identity of indiscernibles) \citep{Bonnotte2013,kolouri2019generalized,nguyen2021distributional,weed2022}. However, to the best of our knowledge, the metric properties of $\swd{\cdot}{\cdot}{\rho}$ for \emph{any} $\rho \in \calP(\Sphere)$ have not been established.

By adapting the proof techniques in \citep{Bonnotte2013,kolouri2019generalized}, and due to the metric properties of the Wasserstein distance, one can show that symmetry, positivity and triangle inequality hold for any $\rho \in \calP(\Sphere)$, and that for any $\mu \in \calP_p(\R^d)$, $\swd{\mu}{\mu}{\rho} = 0$.

However, the reverse implication of the identity of indiscernibles, \ie
\begin{equation} \label{eq:indisc}
    \forall \mu,\nu \in \calP_p(\R^d),\; \swd{\mu}{\nu}{\rho} = 0 \text{\; implies \;} \mu = \nu \,,
\end{equation} 
does not hold for any $\rho \in \calP(\Sphere)$. For example, consider $\mu, \nu \in \calP_p(\mathsf{X})$ with $\mathsf{X} \subset \R^d$, and $\mu$ different from $\nu$. Suppose that $\rho \in \calP(\Theta)$ with $\Theta \subset \Sphere$ such that $\forall (\theta, x) \in \Theta \times \mathsf{X}, \ps{\theta}{x} = 0$. In that case, for any $\theta \sim \rho$, $\thss \mu = \thss \nu = \updelta_{\{0\}}$, and since $\text{W}_p$ is a metric, $\text{W}_p(\thss\mu,\thss\nu) = 0$. Therefore, $\swdp{\mu}{\nu}{\rho} = \int_{\Theta} \wdp{\thss\mu}{\thss\nu} \dd \rho(\theta) = 0$, but $\mu \neq \nu$, so \eqref{eq:indisc} is not satisfied.

We conclude that for any $\rho \in \calP(\Sphere)$, $\swd{\cdot}{\cdot}{\rho}$ is a \emph{pseudo-metric}, and if \eqref{eq:indisc} is satisfied, then it is a metric.

\subsection{Generalization Bounds for SW} \label{subsec:genbound_maxsw}

We precise our argument in \Cref{sec:intro}, 
which states that bounds on the generalization gap for SW distances can be established using existing results for max-SW.

Let $\rho \in \calP(\Sphere)$. By applying the triangle inequality for $\swd{\cdot}{\cdot}{\rho}$, then by the definition of max-SW, we obtain,
\begin{align}
    \E | \swd{\mu_n}{\nu_n}{\rho} - \swd{\mu}{\nu}{\rho} | &\leq \E[\swd{\mu_n}{\mu}{\rho}] + \E[\swd{\nu_n}{\nu}{\rho}] \label{eq:sum_max_sw_0} \\
    &\leq \E[\text{maxSW}(\mu_n, \mu)] + \E[\text{maxSW}(\nu_n, \nu)] \,, \label{eq:sum_max_sw}
\end{align}
where $\E$ is taken with respect to $\{x_i\}_{i=1}^n, \{y_i\}_{i=1}^n$ \iid~from $\mu, \nu$ respectively.
We can then bound from above \eqref{eq:sum_max_sw}, using the convergence rates established in \citep[Section 3.2]{lin2021} or \citep[Theorem 1]{weed2022}. 
These rates vary depending on the properties of $\mu, \nu$: for instance, \citep[Theorem 3.5]{lin2021} holds if $\mu, \nu$ satisfy the Bernstein condition.

Nevertheless, we argue that the generalization bounds resulting from eq.\eqref{eq:sum_max_sw_0}-\eqref{eq:sum_max_sw} are not tight for an arbitrary $\rho \in \calP(\Sphere)$. For instance, since we bound (\ref{eq:sum_max_sw}) with \citep{lin2021,weed2022}, we obtain convergence rates that linearly depend on $d$ for any $\rho$, due to the properties of maximum SW. However, if we consider $\rho = \calU(\Sphere)$, it is known that $ \E | \swd{\mu_n}{\nu_n}{\rho} - \swd{\mu}{\nu}{\rho} |$ converges to 0 at a dimension-free rate \citep{nadjahi2020statistical}.

Another important drawback of such bounds is that the impact of $\rho$ on the convergence rates is unclear. In \Cref{subsec:proof_genbound}, we will further explain why our generalization bounds derived from PAC-Bayesian theory are more flexible and informative for arbitrary $\rho$.

\section{Postponed Proofs for \Cref{sec:pacbound}} \label{appendix:postponed_proofs}


\subsection{Proof of \Cref{thm:generic_bound}}\label{subsec:proof_genbound}

\Cref{thm:generic_bound} is obtained by adapting the proof techniques of Catoni's PAC-Bayesian bound \citep{catoni03pacbayesian}.  
First, we recall Donsker and Varadhan’s variational formula, which plays a central role in the PAC-Bayesian framework. 

\begin{lemma}[Donsker and Varadhan’s variational formula \cite{donsker1975variational}] \label{lem:donskerv}
Let $\Theta$ be a set equipped with a $\sigma$-algebra and $\pi \in \calP(\Theta)$. For any measurable, bounded function $h:\Theta\rightarrow\R$,
\begin{equation}
\log\mathbb{E}_{\theta\sim\pi}\big[\exp(h(\theta))\big] = \underset{\rho \in \calP(\Theta)}{\sup} \big[\E_{\theta\sim\rho}[h(\theta)]-\kld{\rho}{\pi}\big]
\end{equation}
\end{lemma}

\begin{proof}[Proof of \Cref{thm:generic_bound}]
    Let $p \in [1, \pinf)$ and $\mu, \nu \in \calP_p(\R^d)$. Assume there exists $\varphi_{\mu, \nu, p}$ such that for any $\theta \in \Sphere$ and $\lambda > 0$, 
    \begin{equation} \label{supp:eq:momentgenfn}
        \E_{\mu, \nu} \left[ \exp \left( \lambda \Big\{ \wdp{\thss \mu_n}{\thss \nu_n} - \E_{\mu, \nu}[\wdp{\thss \mu_n}{\thss \nu_n}] \Big\} \right) \right] \leq \exp(\lambda^2 \varphi_{\mu, \nu, p} n^{-1}) \eqsp.
    \end{equation}
    
    Let $\rho_0 \in \calP(\Sphere)$. By taking the expectation of \eqref{supp:eq:momentgenfn} with respect to $\rho_0$, then using Fubini's theorem to interchange the expectation over $\rho_0$ and the one over $\mu, \nu$, we obtain
    \begin{equation} \label{supp:eq:fubini}
        \E_{\mu, \nu} \E_{\theta \sim \rho_0} \left[ \exp\big( \lambda \big\{ \wdp{\thss \mu_n}{\thss \nu_n} - \E_{\mu, \nu}[\wdp{\thss \mu_n}{\thss \nu_n}] \big\} \big) \right] \leq \exp(\lambda^2 \varphi_{\mu, \nu, p}n^{-1}) \eqsp.
    \end{equation}
    
    By definition of the Wasserstein distance between empirical, univariate distributions of \eqref{eq:wasserstein_1D_discrete}
    , one can prove that $\theta \mapsto \lambda \big\{ \wdp{\thss \mu_n}{\thss \nu_n} - \E_{\mu, \nu}[\wdp{\thss \mu_n}{\thss \nu_n}] \big\}$ is a bounded real-valued function on $\Sphere$. Therefore, we can apply \Cref{lem:donskerv} to rewrite \eqref{supp:eq:fubini} as follows.
    \begin{align}
        &\E_{\mu, \nu} \Big[ \exp\Big(\sup_{\rho \in \calP(\Sphere)} \big[ \E_{\theta\sim\rho} \big[ \lambda \big\{ \wdp{\thss \mu_n}{\thss \nu_n} - \E_{\mu, \nu}[\wdp{\thss \mu_n}{\thss \nu_n}] \big\}\big] - \kld{\rho}{\rho_0} \big] \Big) \Big] \nonumber \\
        &\qquad \leq \exp(\lambda^2 \varphi_{\mu, \nu, p}n^{-1}) \eqsp,
    \end{align}
    which, using the linearity of the expectation along with the definition of SW \eqref{eq:def_sw}
    , is equivalent to
    \begin{align}
     \E_{\mu, \nu} \Big[ \exp\Big(\sup_{\rho \in \calP(\Sphere)}  \big[ \lambda \big\{ \swdp{\mu_n}{\nu_n}{\rho} - \E_{\mu, \nu}[\swdp{\mu_n}{\nu_n}{\rho}] \big\} - \kld{\rho}{\rho_0} \big]\Big) \Big] \leq \exp(\lambda^2 \varphi_{\mu, \nu, p}n^{-1})\eqsp,\nonumber\\
     \intertext{or,}
    \label{supp:eq:fubini_bis}
        \E_{\mu, \nu} \Big[ \exp\Big(\sup_{\rho \in \calP(\Sphere)}  \big[ \lambda \big\{ \swdp{\mu_n}{\nu_n}{\rho} - \E_{\mu, \nu}[\swdp{\mu_n}{\nu_n}{\rho}] \big\} - \kld{\rho}{\rho_0} \big] - \lambda^2 \varphi_{\mu, \nu, p}n^{-1} \Big) \Big] \leq 1.
    \end{align} 
    
    Let $s > 0$. By the Chernoff bound $\left(\mathbb{P}(X>a) =\mathbb{P}(e^{s.X}\geq e^{s.a}) \leq \E[e^{t.X}]e^{-t.a}\right)$
    \begin{align}
        &\mathbb{P}_{\mu,\nu}\Big(\sup_{\rho \in \calP(\Sphere)}  \big[ \lambda \big\{ \swdp{\mu_n}{\nu_n}{\rho} - \E_{\mu, \nu}[\swdp{\mu_n}{\nu_n}{\rho}] \big\} - \kld{\rho}{\rho_0} \big] - \lambda^2 \varphi_{\mu, \nu, p}n^{-1} > s \Big) \\
        &\leq \E_{\mu,\nu}\Big[ \exp\Big( \sup_{\rho \in \calP(\Sphere)}  \big[ \lambda \big\{ \swdp{\mu_n}{\nu_n}{\rho} - \E_{\mu, \nu}[\swdp{\mu_n}{\nu_n}{\rho}] \big\} - \kld{\rho}{\rho_0} \big] - \lambda^2 \varphi_{\mu, \nu, p}n^{-1} \Big) \Big] \exp(-s) \\
        &\leq 1\cdot \exp(-s) =  \exp(-s) \eqsp,
    \end{align}
    where the last inequality follows from \eqref{supp:eq:fubini_bis}. 
    
    Let $e^{-s} = \varepsilon$ such that $s = \log(1/\varepsilon)$. Then,
    \begin{equation} \label{supp:eq:proba}
        \mathbb{P}_{\mu,\nu}\Big(\exists \rho \in \calP(\Sphere),\; \lambda \big\{ \swdp{\mu_n}{\nu_n}{\rho} - \E_{\mu, \nu}[\swdp{\mu_n}{\nu_n}{\rho}] \big\} - \kld{\rho}{\rho_0} - \lambda^2 \varphi_{\mu, \nu, p}n^{-1} > \log(1/\varepsilon) \Big) \leq \varepsilon \eqsp.
    \end{equation} 
    Taking the complement of \eqref{supp:eq:proba} and rearranging the terms yields
    \begin{align} \label{supp:eq:comp_proba}
        &\mathbb{P}_{\mu,\nu}\Big(\forall \rho \in \calP(\Sphere),\; \swdp{\mu_n}{\nu_n}{\rho} < \E_{\mu, \nu}[\swdp{\mu_n}{\nu_n}{\rho}] + \lambda^{-1} \big\{ \kld{\rho}{\rho_0} +  \log(1/\varepsilon) \big\} + \lambda \varphi_{\mu, \nu, p}n^{-1} \Big) \nonumber \\ 
        &\quad \geq 1 - \varepsilon \eqsp.
    \end{align}
    
    Our final bound results from assuming there exists $\psi_{\mu, \nu, p}(n)$ such that,
    \begin{equation}
        \E_{\mu, \nu}| \swdp{\mu_n}{\nu_n}{\rho} - \swdp{\mu}{\nu}{\rho}| \leq \psi_{\mu, \nu, p}(n) \eqsp.
    \end{equation}
    
\end{proof}

\paragraph{Comparison with \Cref{subsec:genbound_maxsw}.} In our work, instead of bounding $\swdp{\cdot}{\cdot}{\rho}$ by maxSW, we apply PAC-Bayesian theory directly on $\swdp{\cdot}{\cdot}{\rho}$ for any $\rho$. As a result, our PAC-Bayes-inspired bounds are more flexible than bounds in \Cref{subsec:genbound_maxsw}, since their convergence rates adapt to the distribution $\rho$ (via the KL divergence). However, when $\rho$ is a Dirac measure, \Cref{thm:generic_bound} become vacuous because of the KL term, as with most PAC-Bayesian bounds. In such cases, which include maxSW, the bounds in \Cref{subsec:genbound_maxsw} are more informative.

As discussed in \Cref{subsec:bernstein}, in specific settings, $\varphi_{\mu, \nu, p}$ can be a function of $\lambda \in \R_+$ and $n \in \nsets$. In that case, a straightforward adaptation of the proof of \Cref{thm:generic_bound} yields \Cref{thm:extension_generic_bound}, which will be leveraged for distributions with Bernstein-type moment conditions (\Cref{def:bernstein_cond}).

\begin{theorem} \label{thm:extension_generic_bound}
    Let $p \in [1, \pinf)$ and $\mu, \nu \in \calP_p(\R^d)$. 
    Let $\Lambda \subset \R_+^*$ and assume there exists $\varphi_{\mu, \nu, p} : \Lambda \times n \to \R_+$, possibly depending on $\mu, \nu$ and $p$ such that: $\forall \lambda \in \Lambda$, $\forall \theta \in \Sphere$,
    \begin{equation}
        \E \left[ \exp \left( \lambda \big\{ \mathrm{W}_p^p(\thss \mu_n, \thss \nu_n) - \E[\mathrm{W}_p^p(\thss \mu_n, \thss \nu_n)]
        \big\} \right) \right] \leq \exp(\lambda^2 \varphi_{\mu, \nu, p}(\lambda, n) \,n^{-1}) \eqsp,
    \end{equation}
    where $\E$ is taken with respect to the support points of $\mu_n$ and $\nu_n$. Additionally, assume there exists $\psi_{\mu, \nu, p} : \nsets \to \R_+$, possibly depending on $\mu, \nu$ and $p$, such that, $\forall \rho \in \calP(\Sphere)$,
    \begin{equation}
        \E \big| \mathrm{SW}_p^p(\mu_n,\nu_n;\rho)- \mathrm{SW}_p^p(\mu,\nu;\rho)
        \big| \leq \psi_{\mu, \nu, p}(n) \,.
    \end{equation}
    Let $\rho_0 \in \calP(\Sphere)$. Then, for any $\delta \in (0, 1)$, the following holds with probability at least $1 - \delta$: $\forall \rho \in \calP(\Sphere)$,
    \begin{align}
    \mathrm{SW}_p^p(\mu,\nu;\rho) &\geq
    \mathrm{SW}_p^p(\mu_n,\nu_n;\rho) - \frac{\lambda}{n} \varphi_{\mu, \nu, p}(\lambda, n) \\
    &- \frac1{\lambda} \Big\{ \kld{\rho}{\rho_0} + \log\Big(\frac1{\delta}\Big)\Big\} - \psi_{\mu, \nu, p}(n) \,.
    \end{align}
\end{theorem}

\subsection{Proof of \Cref{prop:momgenfn_compact}}
\label{subsec:proof_momgenfn_compact}

To prove \Cref{prop:momgenfn_compact}, we leverage a concentration result that appears in the proof of McDiarmid's inequality (recalled in \Cref{thm:mcdiarmid}), and which relies on the \emph{bounded differences property}  (\Cref{def:bounded_diff}).

\begin{definition}[Bounded differences property] \label{def:bounded_diff}
    Let $\setX \subset \R^d$, $n \in \nsets$ and $c = \{c_i\}_{i=1}^n \in \R^n$. A mapping $f : \setX^n \to \R$ is said to satisfy the \emph{$c$-bounded differences property} if for $i \in \{1, \dots, n\}$, $\{x_i\}_{i=1}^n \in \setX^n$ and $x' \in \setX$, 
    \begin{equation}
        | f(x_1, \dots, x_n) - f(x_1, \dots, x_{i-1}, x', x_{i+1}, \dots, x_n) | \leq c_i \eqsp.
    \end{equation}
\end{definition}

\begin{theorem}[\citep{mcdiarmid1989method}] \label{thm:mcdiarmid}
    Let $(X_i)_{i=1}^n$ be a sequence of $n \in \nsets$ independent random variables with $X_i$ valued in $\setX \subset \R^d$ for $i \in \{1, \dots, n\}$. Let $c = \{c_i\}_{i=1}^n \in \R^n$ and $f : \setX^n \to \mathbb{R}$ satisfying the $c$-bounded differences property. Then, for any $\lambda > 0$,
    \begin{equation}
        \E \big[ \exp(\lambda \{ f - \E[f]\}) \big] \leq \exp(\lambda^2 \|c\|^2/ 8) \eqsp.
    \end{equation}
\end{theorem}

The proof of \Cref{prop:momgenfn_compact} 
consists in applying \Cref{thm:mcdiarmid} to a specific choice of $f$. To this end, we first show that the Wasserstein distance between univariate distributions satisfies the bounded differences property, assuming bounded supports. 

\begin{lemma} \label{lem:bounded_diff_wass}
Let $\setX \subset \R$ be a bounded set with diameter $\Delta = \sup_{(x,x') \in \setX^2} \|x-x'\| < \pinf$. Then, the mapping $f : (\setX^2)^n \to \R_+$ defined for $w_{1:n} \doteq \{(u_i, v_i)\}_{i=1}^n \in (\setX^2)^n$ as 
\begin{equation} \label{eq:f_bounded_diff}
    f(w_{1:n}) = \wdp{\tilde{\mu}_n}{\tilde{\nu}_n} \,
\end{equation}
where $\tilde{\mu}_n, \tilde{\nu}_n$ are the univariate empirical measures computed over $\{u_i\}_{i=1}^n, \{v_i\}_{i=1}^n$ respectively, satisfies the $c$-bounded differences property with $c_i = 2\Delta^p / n$ for $i \in \{1, \dots, n\}$. 

\end{lemma}

\begin{proof}
    For clarity purposes, we start by introducing some notations. Let $n \in \nsets$ and $w_{1:n} \doteq \{(u_j, v_j)\}_{j=1}^n \in (\setX^2)^n$. Denote by $\tilde{\mu}_n, \tilde{\nu}_n$ the empirical distributions supported over $(u_j)_{j=1}^n, (v_j)_{j=1}^n \in \setX^n$ respectively. Let $(u', v') \in \setX^2$ and $i \in \{1, \dots, n\}$. Denote by $\tilde{\mu}'_n$ the empirical distribution supported on $(u'_j)_{j=1}^n$ where $u'_j = u'$ if $j = i$, $u'_j = u_j$ otherwise, and by $\tilde{\nu}'_n$ the empirical distribution over $(v'_j)_{j=1}^n$ where $v'_j = v'$ if $j = i$, $v'_j = v_j$ otherwise.
    
    By definition of the Wasserstein distance between univariate distributions \eqref{eq:wasserstein_1D_discrete},
    \begin{align}
        \wdp{\tilde{\mu}_n}{\tilde{\nu}_n} - \wdp{\tilde{\mu}'_n}{\tilde{\nu}'_n} &= \frac1n \sum_{j=1}^n | u_{\sigma(j)} - v_{\tau(j)} |^p - \frac1n \sum_{j=1}^n | u'_{\sigma'(j)} - v'_{\tau'(j)} |^p
    \end{align}
    where $\sigma : \{1, \dots, n\} \to \{1, \dots, n\}$ (respectively, $\sigma' : \{1, \dots, n\} \to \{1, \dots, n\}$) is the permutation s.t. for $j \in \{1, \dots, n\}$, $u_{\sigma(j)}$ (resp., $u'_{\sigma'(j)}$) is the $j$-th smallest value of $(u_j)_{j=1}^n$ (resp., $(u'_j)_{j=1}^n$). Let $\tau : \{1, \dots, n\} \to \{1, \dots, n\}$ (respectively, $\tau' : \{1, \dots, n\} \to \{1, \dots, n\}$) s.t. for $j \in \{1, \dots, n\}$, $v_{\tau(j)}$ (resp., $v'_{\tau'(j)}$) is the $j$-th smallest value of $(v_j)_{j=1}^n$ (resp., $(v'_j)_{j=1}^n$).

    Therefore,
    \begin{align}
        \wdp{\tilde{\mu}_n}{\tilde{\nu}_n} - \wdp{\tilde{\mu}'_n}{\tilde{\nu}'_n} &\leq \frac1n \sum_{j=1}^n | u_{\sigma'(j)} - v_{\tau'(j)} |^p - \frac1n \sum_{j=1}^n | u'_{\sigma'(j)} - v'_{\tau'(j)} |^p \\
        &= \frac1n \left( |u_i - v_{\tau' \circ \sigma'^{-1}(i)}|^p - |u' - v'_{\tau' \circ \sigma'^{-1}(i)}|^p + |u_{\sigma' \circ \tau'^{-1}(i)} - v_i|^p - |u'_{\sigma' \circ \tau'^{-1}(i)} - v'|^p \right) \\
        &\leq \frac{2\Delta^p}n 
    \end{align}
    We can use the same arguments to prove that $\wdp{\tilde{\mu}'_n}{\tilde{\nu}'_n} - \wdp{\tilde{\mu}_n}{\tilde{\nu}_n} \leq 2\Delta^p / n$. We conclude that,
    \begin{equation} \label{eq:bound_diff_oneside}
        \left| \wdp{\tilde{\mu}_n}{\tilde{\nu}_n} - \wdp{\tilde{\mu}'_n}{\tilde{\nu}'_n} \right| \leq \frac{2\Delta^p}n \eqsp.
    \end{equation}
    

\end{proof}

\begin{remark}
 \Cref{lem:bounded_diff_wass} is an extension of \citep[Proposition 20]{weedbach2019}, which establishes a concentration bound for $\wdp{\mu}{\mu_n}$ around its expectation on any finite-dimensional compact space by exploiting McDiarmid's inequality along with the Kantorovich duality. We thus use similar arguments to prove \Cref{prop:momgenfn_compact}
 , except that we leverage the closed-form expression of the one-dimensional Wasserstein distance instead of the dual formulation since we compare univariate (projected) distributions. 
\end{remark}
 
\begin{proof}[Proof of \Cref{prop:momgenfn_compact}]
Let $\mu, \nu \in \calP(\setX)$ where $\setX \subset \R^d$ has a finite diameter $\Delta$. Let $\theta \in \Sphere$. Then, $\thss\mu, \thss\nu$ are both supported on a bounded domain $\setX_{\theta} \subset \R$ whose diameter is denoted by $\Delta_\theta$ and satisfies $\Delta_\theta \leq \Delta$. Consider the mapping $f$ defined as in \eqref{eq:f_bounded_diff}. 
Given \Cref{lem:bounded_diff_wass}, we can apply \Cref{thm:mcdiarmid} to bound the moment-generating function of $f - \E f$: for any $\lambda > 0$,
\begin{align}
    \E \big[ \exp(\lambda \{ f - \E[f]\}) \big] &\leq \exp(\lambda^2 \sum_{i=1}^n (2\Delta_\theta^p / n)^2 / 8) \\
    &\leq \exp(\lambda^2 \Delta_\theta^{2p} / (2n)) \leq \exp(\lambda^2 \Delta^{2p} / (2n)) \eqsp,
\end{align}
where the expectation is computed over $n$ samples $w_{1:n} \doteq \{(u_i, v_i)\}_{i=1}^n \in (\setX_\theta^2)^n$ \iid~from $\thss \mu \times \thss \nu$. We conclude by using the property of push-forward measures, which gives
\begin{align} \label{eq:exp_pushforward_reformulation}
    \E_{w_{1:n} \sim (\thss \mu \times \thss \nu)^n}\big[ \exp(\lambda \{ f(w_{1:n}) - \E[f(w_{1:n})]\}) \big] = \E_{z_{1:n} \sim (\mu \times \nu)^n} \big[ \exp(\lambda \{ f(\theta^*(z'_{1:n})) - \E[f(\theta^*(z'_{1:n}))]\}) \big]
\end{align}
where for $z_{1:n} \doteq \{(x_i, y_i)\}_{i=1}^n \in (\setX^2)^n$, $\theta^*(z_{1:n}) \doteq \{(\ps{\theta}{x_i}, \ps{\theta}{y_i})\}_{i=1}^n \in (\setX_\theta^2)^n$.

\end{proof}

\subsection{Proof of \Cref{prop:splcpx_compact}}
\label{subsec:proof_splcpx_compact}

Recent work have bounded $\E|\swd{\mu_n}{\nu_n}{\rho} - \swd{\mu}{\nu}{\rho}|$ or $\E|\swd{\mu}{\mu_n}{\rho}|$ for specific choices of $\rho \in \calP(\Sphere)$ \citep{nadjahi2020statistical,manole2020minimax,nguyen2021distributional,lin2021}. These results do not exactly correspond to what \Cref{thm:generic_bound}
requires, \ie~a bound on $\E|\swdp{\mu_n}{\nu_n}{\rho} - \swdp{\mu}{\nu}{\rho}|$. We bound the latter quantity in \Cref{prop:splcpx_compact}, by specifying the proof techniques in \citep{manole2020minimax} for distributions with bounded supports, then generalizing a result in \citep{nadjahi2020statistical}.

\begin{lemma} \label{lem:splcpx_wass1d}
    Let $\setX \subset \R$ be a bounded set whose diameter is denoted by $\Delta < \pinf$. Let $\mu, \nu \in \calP(\setX)$ and denote by $\mu_n, \nu_n$ the empirical distributions supported over $n \in \nsets$ samples i.i.d. from $\mu, \nu$ respectively. Let $p \in [1, \pinf)$. Then, there exists a constant $C$ such that,
    \begin{equation}
        \E\big| \wdp{\mu_n}{\nu_n} - \wdp{\mu}{\nu} \big|
        \leq C p \Delta^p n^{-1/2} \eqsp.
    \end{equation}
\end{lemma}

\begin{proof}
    \Cref{lem:splcpx_wass1d} is obtained by adapting the techniques used in the proof of \citep[Lemma 6]{manole2020minimax}, then applying \citep[Theorem 1]{fournier2015rate}. We provide the detailed proof for completeness.

    Starting from the definition of $\wdp{\mu_n}{\nu_n}$ \eqref{eq:wasserstein_1D_discrete}, then using a Taylor expansion of $(x,y) \mapsto |x-y|^p$ around $(x,y) = (F_{\mu}^{-1}(t), F_{\nu}^{-1}(t))$, we obtain 
    \begin{align}
        \wdp{\mu_n}{\nu_n} &= \int_0^1 \big| F_{\mu_n}^{-1}(t) - F_{\nu_n}^{-1}(t) \big|^p \dd t \nonumber \\
        &= \int_0^1 \big| F_{\mu}^{-1}(t) - F_{\nu}^{-1}(t) \big|^p \dd t \label{eq:taylor1} \\
        &\; + \int_0^1 p~\text{sgn}\big( \tilde{F}_{\mu_n}^{-1}(t) - \tilde{F}_{\nu_n}^{-1}(t) \big) \big| \tilde{F}_{\mu_n}^{-1}(t) - \tilde{F}_{\nu_n}^{-1}(t) \big|^{p-1} \big\{ (F_{\mu_n}^{-1}(t) - F_{\mu}^{-1}(t)) - (F_{\nu_n}^{-1}(t) - F_{\nu}^{-1}(t)) \big\} \dd t \label{eq:taylor2}
    \end{align}
    where $\text{sgn}(\cdot)$ denotes the sign function, $\tilde{F}_{\mu_n}^{-1}(t)$ a real number between $F_{\mu_n}^{-1}(t)$ and $F_{\mu}^{-1}(t)$, and $\tilde{F}_{\nu_n}^{-1}(t)$ a real number between $F_{\nu_n}^{-1}(t)$ and $F_{\nu}^{-1}(t)$. 
    
    By definition, \eqref{eq:taylor1} is exactly $\wdp{\mu}{\nu}$ and we obtain
    \begin{align}
        &| \wdp{\mu_n}{\nu_n} - \wdp{\mu}{\nu} | \nonumber \\
        &= \Big| \int_0^1 p~\text{sgn}\big( \tilde{F}_{\mu_n}^{-1}(t) - \tilde{F}_{\nu_n}^{-1}(t) \big) \big| \tilde{F}_{\mu_n}^{-1}(t) - \tilde{F}_{\nu_n}^{-1}(t) \big|^{p-1} \big\{ (F_{\mu_n}^{-1}(t) - F_{\mu}^{-1}(t)) - (F_{\nu_n}^{-1}(t) - F_{\nu}^{-1}(t)) \big\} \dd t \Big| \nonumber \\
        &\leq p \int_0^1 \big| \tilde{F}_{\mu_n}^{-1}(t) - \tilde{F}_{\nu_n}^{-1}(t) \big|^{p-1} \Big\{ \big| F_{\mu_n}^{-1}(t) - F_{\mu}^{-1}(t) \big| + \big| F_{\nu_n}^{-1}(t) - F_{\nu}^{-1}(t) \big| \Big\} \dd t \label{eq:triangleineq} \\
        &\leq p \sup_{t \in (0, 1)} \big| \tilde{F}_{\mu_n}^{-1}(t) - \tilde{F}_{\nu_n}^{-1}(t) \big|^{p-1} \Big\{ \wdone{\mu_n}{\mu} + \wdone{\nu_n}{\nu} \Big\} \eqsp, \label{eq:sum_wdone}
    \end{align}
    where \eqref{eq:triangleineq} follows from the triangle inequality and \eqref{eq:sum_wdone} results from the definition of the Wasserstein distance of order 1 between univariate distributions. 

    We then bound $\sup_{t \in (0, 1)} \big| \tilde{F}_{\mu_n}^{-1}(t) - \tilde{F}_{\nu_n}^{-1}(t) \big|^{p-1}$ from above. By the definition of $\tilde{F}_{\mu_n}^{-1}(t), \tilde{F}_{\nu_n}^{-1}(t)$ for $t \in (0,1)$, we distinguish the following four cases:
    \begin{enumerate}[label=(\roman*)]
    \itemsep0em 
        \item $\tilde{F}_{\mu_n}^{-1}(t) \leq F_{\mu_n}^{-1}(t)$, $\tilde{F}_{\nu_n}^{-1}(t) \leq F_{\nu_n}^{-1}(t)$
        \item $\tilde{F}_{\mu_n}^{-1}(t) \leq F_{\mu}^{-1}(t)$, $\tilde{F}_{\nu_n}^{-1}(t) \leq F_{\nu}^{-1}(t)$
        \item $\tilde{F}_{\mu_n}^{-1}(t) \leq F_{\mu_n}^{-1}(t)$, $\tilde{F}_{\nu_n}^{-1}(t) \leq F_{\nu}^{-1}(t)$
        \item $\tilde{F}_{\mu_n}^{-1}(t) \leq F_{\mu}^{-1}(t)$, $\tilde{F}_{\nu_n}^{-1}(t) \leq F_{\nu_n}^{-1}(t)$
    \end{enumerate}
    Hence, using the definition of quantile functions and the fact that the supports of $\mu, \nu$ are assumed to be bounded, we obtain
    \begin{equation}
        \sup_{t\in(0, 1)} \big| \tilde{F}_{\mu_n}^{-1}(t) - \tilde{F}_{\nu_n}^{-1}(t) \big|^{p-1} \leq \Delta^{p-1} \eqsp. \label{eq:bound_sup}
    \end{equation}
    We conclude that,
    \begin{align}
        | \wdp{\mu_n}{\nu_n} - \wdp{\mu}{\nu} | \leq p \Delta^{p-1} \Big\{ \wdone{\mu_n}{\mu} + \wdone{\nu_n}{\nu} \Big\} \eqsp. 
    \end{align}
    and by linearity of the expectation,
    \begin{align}
        \E | \wdp{\mu_n}{\nu_n} - \wdp{\mu}{\nu} | \leq p \Delta^{p-1} \Big\{ \E[\wdone{\mu_n}{\mu}] + \E[\wdone{\nu_n}{\nu}] \Big\} \eqsp. \label{eq:exp_wass_diff}
    \end{align}
    
    Our final result follows from applying \citep[Theorem 1]{fournier2015rate}. Since $\mu, \nu \in \calP(\setX)$ where $\setX \subset \R$ is a bounded set with finite diameter $\Delta < \infty$, then for any $q \geq 1$, the moment of $\mu$ (or $\nu$) of order $q$ is bounded by $\Delta^q$. Therefore, the application of \citep[Theorem 1]{fournier2015rate} yields,
    \begin{align}
        \E [ \wdone{\mu_n}{\mu} ] &\leq C' \Delta n^{-1/2} \eqsp, \;\; \E [ \wdone{\nu_n}{\nu} ] \leq C' \Delta n^{-1/2} \eqsp. \label{eq:fournier}
    \end{align}
    where $C'$ is a constant. We conclude by plugging \eqref{eq:fournier} in \eqref{eq:exp_wass_diff}.
    
    \end{proof}

\begin{proof}[Proof of \Cref{prop:splcpx_compact}]
    Let $\theta \in \Sphere$. Since we assume that $\mu, \nu \in \calP(\setX)$ where $\setX \subset \R^d$ is a bounded subset with finite diameter $\Delta$, one can easily prove that $\thss \mu, \thss\nu$ are supported on a bounded domain with diameter $\Delta_\theta \leq \Delta < \pinf$. Therefore, by \Cref{lem:splcpx_wass1d}, there exists a constant $C$ such that,
    \begin{align} \label{eq:wass_proj_complexity}
        \E\big| \wdp{\thss \mu_n}{\thss \nu_n} - \wdp{\thss \mu}{\thss \nu} \big| \leq C p \Delta^p n^{-1/2} \eqsp. 
    \end{align}
    
    Next, we adapt the proof techniques in \citep[Theorem 4]{nadjahi2020statistical} to establish the following inequality: for \emph{any} $\rho \in \calP(\Sphere)$,
    \begin{align} \label{eq:sw_cpx}
        \E | \swdp{\mu_n}{\nu_n}{\rho} - \swdp{\mu}{\nu}{\rho} | &\leq \int_{\Sphere} \E | \wdp{\thss \mu_n}{\thss \nu_n} - \wdp{\thss \mu}{\thss \nu} | \dd \rho(\theta) \eqsp.
    \end{align}
    Hence, by plugging \eqref{eq:wass_proj_complexity} in \eqref{eq:sw_cpx}, we obtain
    \begin{align}
        \E | \swdp{\mu_n}{\nu_n}{\rho} - \swdp{\mu}{\nu}{\rho} | &\leq C p \Delta^p n^{-1/2} \eqsp. 
    \end{align}
\end{proof}

\subsection{Final Bound for Bounded Supports}\label{secapp: finalbound bounded}

By incorporating \Cref{prop:momgenfn_compact,prop:splcpx_compact} in \Cref{thm:generic_bound}, 
we obtain the following result. \Cref{cor:bound_bounded} corresponds to a specialization of our generic bound when considering distributions with bounded supports.

\begin{corollary} \label{cor:bound_bounded}
    Let $p \in [1, \pinf)$ and assume a bounded diameter $\Delta$. Let $\rho_0 \in \calP(\Sphere)$ and $\delta > 0$. Then, with probability at least $1 - \delta$, for all $\rho \in \calP(\Sphere)$ and $\lambda > 0$, there exists a constant $C$ such that,
    \begin{align} 
        \swdp{\mu_n}{\nu_n}{\rho} \leq \swdp{\mu}{\nu}{\rho} + \left\{\kld{\rho}{\rho_0} + \log(1/\delta)\right\} \lambda^{-1} + \lambda \Delta^{2p}(2n)^{-1} + C p \Delta^p n^{-1/2} 
    \end{align}
\end{corollary}

\subsection{Proof of \Cref{prop:momgenfn_subg}} \label{appendix:momgen_subg}

When the supports of the distributions are not bounded, \Cref{lem:bounded_diff_wass} does not hold true, thus preventing the use of McDiarmid's inequality. Hence, to compute $\varphi_{\mu, \nu, p}$, we may use extensions of McDiarmid's inequality which replace the finite-diameter constraint by conditions on the moments of the distributions.

In particular, \Cref{prop:momgenfn_subg}
follows from applying \citep[Theorem 1]{kontorovicha14}, a concentration result based on the notion of \emph{sub-Gaussian diameter}. 

\begin{definition}[Sub-Gaussian diameter \citep{kontorovicha14}] \label{def:subg_diameter} Let $\eta$ be a distance function and $(\setX, \eta, \mu)$ be the associated metric probability space. Consider a sequence of $n \in \nsets$ independent random variables $(X_i)_{i=1}^n$ with $X_i$ distributed from $\mu$ for $i \in \{1, \dots, n\}$. Let $\Xi(\setX)$ be the random variable defined by
\begin{equation}
    \Xi(\setX) = \varepsilon \eta(X, X') \eqsp, 
\end{equation}
where $X, X'$ are two independent realizations from $\mu$ and $\varepsilon$ is a random variable valued in $\{-1, 1\}$ s.t. $p(\varepsilon = 1) = 1/2$ and $\varepsilon$ is independent from $X, X'$. Additionally, suppose there exists $\sigma > 0$ s.t. for $\lambda \in \R$, $\E_\mu [\exp(\lambda X)] \leq \exp(\sigma^2 \lambda^2 / 2)$. The \emph{sub-Gaussian diameter} of $(\setX, \eta, \mu)$, denoted by $\Deltasg(\setX)$, is defined as $\Deltasg(\setX) = \sigma\big(\Xi(\setX)\big)$.
\end{definition}

Note that $\Deltasg \leq \Delta$ \citep[Lemma 1]{kontorovicha14}. Since a set with infinite diameter may have a finite sub-Gaussian diameter, \Cref{thm:kontorovich} relaxes the conditions of \Cref{thm:mcdiarmid}.

\begin{theorem}[Theorem 1 \citep{kontorovicha14}] \label{thm:kontorovich}
    Let $\setX \subset \R^d$ and $\eta : \setX \times \setX \to \R_+$ be a distance function. Consider the metric probability space $(\setX, \eta)$. For $n \in \nsets$, let $\setX^n$ be the product probability space equipped with the product measure $\mu^n = \mu_1 \times \dots \times \mu_n$, where $\mu_i = \mu$. Define the \emph{$L_1$ product metric} $\eta^n$ for any $(x, x') \in \setX^n \times \setX^n$ as,
    \begin{equation}
        \eta^n(x,x') = \sum_{i=1}^n \eta(x_i, x'_i) \eqsp.
    \end{equation}
    Let $f : \setX^n \to \R$ s.t. $f$ is $1$-Lipschitz with respect to $\eta^n$, \ie~for any $(x,x') \in \setX^n \times \setX^n$, $|f(x) - f(x')| \leq \eta(x, x')$. Then, $\E [f] < \pinf$ and for $\lambda > 0$,
    \begin{equation}
        \E \left[ \exp(\lambda \{f - \E[f]\}) \right] \leq \exp(\lambda^2 n \Delta_{SG}(\setX)^2 / 2) \eqsp.
    \end{equation}
\end{theorem}

As discussed in \citep{kontorovicha14}, the sub-Gaussian distributions on $\R$ are precisely those for which
$\Deltasg(\R) < \pinf$. \Cref{prop:momgenfn_subg} then results from applying \Cref{thm:kontorovich}, as explained below.

\begin{proof}[Proof of \Cref{prop:momgenfn_subg}]
    First, we prove that for any $\mu \in \calP(\R^d)$ such that $\mu$ is sub-Gaussian with parameter $\sigma^2$, then $\mu \in \calP_1(\R^d)$. By definition, the first moment of $\mu$ is $\textrm{m}_1(\mu) = \int_{\R^d} \|x\| \dd \mu(x)$.
    For any $x \in \R^d$, we know that
    \begin{equation}
        \|x\| = \left( \sum_{k=1}^d |x_k|^2 \right)^{1/2} \leq \sum_{k=1}^d |x_k|
    \end{equation}
    Therefore, $\textrm{m}_1(\mu)$ can be bounded from above as follows.
    \begin{align}
        \mathrm{m}_1(\mu) &\leq \int_{\R^d} \sum_{k=1}^d |x_k| \dd \mu(x) \\
        &\leq \sum_{k=1}^d \int_{\R^d} |x_k| \dd \mu(x) \label{eq:linearity_exp} \\
        &\leq \sum_{k=1}^d \int_{\R^d} |\ps{\theta^k}{x}| \dd \mu(x) \\
        &\leq \sum_{k=1}^d \int_{\R} |t| \dd(\theta^k)^\star_\sharp \mu(t) \label{eq:pushforward_def} \\
        &\leq d \sqrt{2\pi \sigma^2}\label{eq:subgassumption}
    \end{align}
    where for $k \in \{1, \dots, d\}$, $\theta^k \in \Sphere$ is defined as $(\theta^k)_i = 1$ if $i = k$, $(\theta^k)_i = 0$ otherwise. \eqref{eq:linearity_exp} results from the linearity of the expectation, \eqref{eq:pushforward_def} is obtained by applying the property of pushforward measures. \eqref{eq:subgassumption} follows from the sub-Gaussian assumption on $\mu$ (\Cref{def:subg}) and \citep[Proposition 3.2]{Rivasplata2012SubgaussianRV}. Since $\mathrm{m}_1(\mu) < \infty$ \eqref{eq:subgassumption}, we conclude that $\mu \in \calP_1(\R^d)$.
    
    Now, consider the product metric space $(\R^2, \eta)$ where $\eta : \R^2 \to \R_+$ is the distance function defined for $w \doteq (u,v) \in \R^2$, $w' \doteq (u',v') \in \R^2$ as,
    \begin{equation}
        \eta(w,w') \doteq \|u-u'\| + \|v-v'\| = |u-u'| + |v-v'| \eqsp.
    \end{equation}
    Let $n \in \nsets$ and define $f : (\R^2)^n \to \R_+$ as: for any $w_{1:n} \doteq (w_i)_{i=1}^n \in (\R^2)^n$ such that $\forall i \in \{1, \dots, n\}, w_i = (u_i, v_i) \in \R^2$,
    \begin{equation}
        f(w_{1:n}) = n\,\wdone{\tilde{\mu}_n}{\tilde{\nu}_n} \,, \label{eq:f}
    \end{equation}
    where $\tilde{\mu}_n, \tilde{\nu}_n$ are the empirical distributions computed over $(u_i)_{i=1}^n, (v_i)_{i=1}^n$ respectively, \ie, denoting by $\updelta_x$ the Dirac measure at $x$, \begin{equation}
    \tilde{\mu}_n = \frac1n \sum_{i=1}^n \updelta_{u_i},\quad \tilde{\nu}_n = \frac1n \sum_{i=1}^n \updelta_{v_i}\eqsp.
    \end{equation}
    
    We prove that $f$ is $1$-Lipschitz with respect to the $L_1$ product metric $\eta^{n}$ defined for any $w_{1:n} \doteq \{(u_i, v_i)\}_{i=1}^n \in (\R^2)^n$, $w'_{1:n} \doteq \{(u'_i, v'_i)\}_{i=1}^n \in (\R^2)^n$ as,
    \begin{equation}
        \eta^n(w_{1:n}, w'_{1:n}) = \sum_{i=1}^n \big\{ \| u_i - u'_i \| + \| v_i - v'_i \| \big\} \eqsp. \label{eq:eta_n}
    \end{equation}
    
    Since the Wasserstein distance satisfies the triangle inequality, one has
    \begin{equation}
        |\wdone{\tilde{\mu}_n}{\tilde{\nu}_n} - \wdone{\tilde{\mu}'_n}{\tilde{\nu}'_n} | \leq \wdone{\tilde{\mu}_n}{\tilde{\mu}'_n} + \wdone{\tilde{\nu}_n}{\tilde{\nu}'_n}
    \end{equation} 
    where $\tilde{\mu}_n, \tilde{\nu}_n$  are the empirical distributions supported on $(u_i)_{i=1}^n, (v_i)_{i=1}^n$ respectively, and $\tilde{\mu}'_n, \tilde{\nu}'_n$  are the empirical distributions supported on $(u'_i)_{i=1}^n, (v'_i)_{i=1}^n$ respectively. 
    By definition of the Wasserstein distance between univariate discrete distributions \citep[Remark 2.28]{peyre2019computational},
    \begin{align}
        \wdone{\tilde{\mu}_n}{\tilde{\mu}'_n} &= \frac1n \sum_{i=1}^n \big| u_{(i)} - u'_{(i)} \big| \\
        &\leq \frac1n \sum_{i=1}^n | u_i - u'_i | = \frac1n \sum_{i=1}^n \| u_i - u'_i \| \label{eq:cs_w1}
    \end{align}
    where $u_{(1)} \leq u_{(2)} \leq \dots \leq u_{(n)}$ and $u'_{(1)} \leq u'_{(2)} \leq \dots \leq u'_{(n)}$. Analogously, $\wdone{\tilde{\nu}_n}{\tilde{\nu}'_n} \leq \frac1n \sum_{i=1}^n \| v_i - v'_i \|$. 
    Therefore,
    \begin{align}
        |\wdone{\tilde{\mu}_n}{\tilde{\nu}_n} - \wdone{\tilde{\mu}'_n}{\tilde{\nu}'_n} | &\leq \frac1n \sum_{i=1}^n \big\{ \|u_i - u'_i\| + \|v_i - v'_i\| \big\} \eqsp. \label{eq:wasslip}
    \end{align}
    We conclude from \eqref{eq:f} and \eqref{eq:wasslip} that $f$ is $1$-Lipschitz with respect to the product metric $\eta^n$, as defined in \eqref{eq:eta_n}. 
    

    Next, let $\theta \in \Sphere$ and $\mu, \nu \in \calP(\R^d)$ such that $\mu, \nu$ are sub-Gaussian with respective variance proxy $\sigma^2, \tau^2$. Consider the probability metric space $(\R^2, \eta, \thss \mu \times \thss \nu)$. By \Cref{def:subg_diameter} and the properties of the sum of independent sub-Gaussian random variables \citep[Theorem 2.7]{Rivasplata2012SubgaussianRV}, the sub-Gaussian diameter of that space is $\Deltasg(\R^2) = \sqrt{2 (\sigma^2+\tau^2)}$.
    
    We conclude the proof by applying \Cref{thm:kontorovich} to $f$ as defined in \eqref{eq:f}, then reformulating the expectation over $\thss \mu \times \thss \nu$ as an expectation over $\mu \times \nu$ using the property of push-forward measures (see \eqref{eq:exp_pushforward_reformulation}). 
    
\end{proof}

\subsection{Proof of \Cref{prop:momgenfn_bernstein}}

\Cref{prop:momgenfn_bernstein} results from the same arguments as in the proof of \citep[Corollary 5.2]{lei2020}. The latter result is obtained by applying a generalized McDiarmid's inequality, which we recall in \Cref{thm:bernstein_mcdiarmid}.

\begin{theorem}[Bernstein-type McDiarmid's inequality \citep{lei2020}] \label{thm:bernstein_mcdiarmid}
    Let $\setX \subset \R^d$ and $X = (X_i)_{i=1}^n$ be a sequence of $n \in \mathbb{N}^*$ random variables i.i.d. from $\mu \in \calP(\setX)$. Let $f : \setX^n \to \R$ s.t. $\E |f| < \infty$. For $i \in \{1, \dots, n\}$, let $X'_i$ be an independent copy of $X_i$ and $X'_{(i)} = (X_1, \dots, X_{i-1}, X'_i, X_{i+1}, \dots, X_n)$. Assume for $i \in \{1, \dots, n\}$, there exists $c_i, M > 0$ s.t. for $k \geq 2$,
    \begin{equation} \label{eq:bernstein_tail}
        \E\big[ | f(X) - f(X'_{(i)})|^k~|~X_{-i} \big] \leq c_i^2 k! M^{k-2} / 2 \eqsp,
    \end{equation}
    where $X_{-i} = (X_1, \dots, X_{i-1}, X_{i+1}, \dots, X_n)$. Then, for $\lambda > 0$ s.t. $\lambda M < 1$,
    \begin{equation}
        \E\big[ \exp\{ \lambda (f - \E[f]) \} \big] \leq \exp\left( \lambda^2 \|c\|^2 / \{2(1-\lambda M) \} \right) \eqsp.
    \end{equation}
\end{theorem}

\begin{proof}[Proof of \Cref{prop:momgenfn_bernstein}]
First, we justify why for any $\mu \in \calP(\R^d)$ s.t. $\mu$ satisfies the $(\sigma^2, b)$-Bernstein condition, $\mu \in \calP_1(\R^d)$. By \eqref{eq:pushforward_def}, the first order moment of $\mu$, $\textrm{m}_1(\mu)$ can be bounded as,
\begin{align}
    \textrm{m}_1(\mu) &\leq \sum_{k=1}^d \int_{\R} |t| \dd(\theta^k)^\star_\sharp \mu(t) \\
    &\leq \sum_{k=1}^d \left\{ \int_{\R} |t|^2 \dd(\theta^k)^\star_\sharp \mu(t) \right\}^{1/2} \label{eq:holder} \\
    &\leq d \sigma \label{eq:moment}
\end{align}
where \eqref{eq:holder} is obtained by applying H\"older's inequality, and \eqref{eq:moment} results from \Cref{def:bernstein_cond}. Hence, $\mathrm{m}_1(\mu) < \infty$ and $\mu \in \calP_1(\R^d)$.

The rest of the proof consists in applying \Cref{thm:bernstein_mcdiarmid} to $f : (\R^2)^n \to \R_+$ defined for any $w_{1:n} \doteq \{(u_i, v_i)\}_{i=1}^n \in (\R^2)^n$ as, 
\begin{equation}
    f(w_{1:n}) = \wdone{\tilde{\mu}_n}{\tilde{\nu}_n} \label{eq:f_bernstein}
\end{equation}
where $\tilde{\mu}_n, {\tilde{\nu}_n}$ are the empirical distributions of $(u_i)_{i=1}^n, (v_i)_{i=1}^n$ respectively.

For $i \in \{1, \dots, n\}$, let $(u'_i, v'_i) \in \R^2$. Denote by $\tilde{\mu}'_n$ the empirical distribution supported on $(u_1, \dots, u_{i-1}, u'_i, u_{i+1}, \dots, u_n) \in \R^n$, and by $\tilde{\nu}'_n$ the empirical distribution supported on $(v_1, \dots, v_{i-1}, v'_i, v_{i+1}, \dots, v_n) \in \R^n$. Then,
\begin{align}
    | \wdone{\tilde{\mu}_n}{\tilde{\nu}_n} - \wdone{\tilde{\mu}'_n}{\tilde{\nu}'_n} | &\leq \wdone{\tilde{\mu}_n}{\tilde{\mu}'_n} + \wdone{\tilde{\nu}_n}{\tilde{\nu}'_n} \label{eq:wass_triangle} \\
    &\leq \frac1n \big\{ |u_i - u'_i| + \sum_{j=1, \dots, n, j \neq i} | u_j - u_j |  \big\} + \frac1n \big\{ |v_i - v'_i| + \sum_{j=1, \dots, n, j \neq i} | v_j - v_j |  \big\} \quad \label{eq:bound_wass1D} \\
    &\leq \frac1n \big\{ | u_i - u'_i | + |v_i - v'_i| \big\} \label{eq:bernstein_proof_1}
\end{align}
where \eqref{eq:wass_triangle} follows from the fact that $\text{W}_1$ satisfies the triangle inequality, and \eqref{eq:bound_wass1D} results from the definition of the Wasserstein distance between univariate empirical distributions \citep[Remark 2.28]{peyre2019computational}. 


Now, let $\mu \in \calP(\R^d)$ (respectively, $\nu \in \calP(\R^d)$) satisfy the $(\sigma^2, b)$ (resp., $(\tau^2, c)$)-Bernstein condition (\Cref{def:bernstein_cond}). Let $\theta \in \Sphere$ and consider $w_{1:n} = \{(u_i, v_i)\}_{i=1}^n \in (\R^2)^n$ i.i.d. from the product measure $\thss \mu \times \thss \nu$. We justify why $f$ satisfies the conditions of \Cref{thm:bernstein_mcdiarmid}.

First, we show that $\mathbb{E}|f|$ is finite, where the expectation $\mathbb{E}$ is computed over $n$ \iid~samples $\{(u_i, v_i)\}_{i=1}^n$ from $\thss \mu \times \thss \nu_n$.
\begin{align}
    \E|f| &\leq \frac1n  \sum_{i=1}^n\E[|u_i - v_i|] \leq \frac1n \sum_{i=1}^n  \left\{ \E|u_i| + \E|v_i| \right\} \\  
    &\leq \frac1n \sum_{i=1}^n  \left\{ \E[|u_i|^2]^{1/2} + \E[|v_i|^2]^{1/2} \right\} \label{eq:holder_2} \\
    &\leq \sigma + \tau \label{eq:expectation_finite}
\end{align}
where \eqref{eq:holder_2} results from H\"older's inequality, and \eqref{eq:expectation_finite} directly follows from the definition of the Bernstein condition (\Cref{def:bernstein_cond}).

Besides, by using \eqref{eq:bernstein_proof_1} and the Bernstein condition \Cref{def:bernstein_cond}, one can show that
\begin{equation}
    \E [ | \wdone{\tilde{\mu}_n}{\tilde{\nu}_n} - \wdone{\tilde{\mu}'_n}{\tilde{\nu}'_n} |^k~|~u_{-i}, v_{-i} ] \leq n^{-k} 2^{2(k-1)} [\sigma^2 b^{k-2} + \tau^2 c^{k-2}] k!
\end{equation}
where the expectation is computed over $\{(u_i, v_i)\}_{i=1}^n$ i.i.d. from $\thss \mu \times \thss \nu$. In other words, $f$ as defined in \eqref{eq:f_bernstein} satisfies \eqref{eq:bernstein_tail} with, for $i \in \{1, \dots, 2n\}$, $c_i = 2 \sigma_\star n^{-1}$ and $M = 4 b_\star n^{-1}$, where $\sigma_\star = \max(\sigma, \tau)$ and $b_\star = \max(b, c)$. Our final result follows from applying  \Cref{thm:bernstein_mcdiarmid} to $f$, then applying the property of push-forward measures to obtain the expectation with respect to $\mu \times \nu$ (see \eqref{eq:exp_pushforward_reformulation}). 

\end{proof}
\vspace{-1em}
\subsection{Final Bound for Unbounded Supports}\label{secapp: final bound unbounded}

Before deriving the specialization of \Cref{thm:generic_bound}
for distributions with unbounded supports, we recall a useful bound on $\swdp{\cdot}{\cdot}{\pi}$ with $\pi = \calU(\Sphere)$ (\Cref{thm:manole}), which can be generalized for SW based on any $\rho \in \calP(\Sphere)$ by adapting the proof techniques in \citep[Theorem 2]{manole2020minimax}.

\begin{theorem}[\cite{manole2020minimax}] \label{thm:manole}
Let $p \geq 1$, $q > 2p$, $s \geq 1$ and $\pi = \calU(\Sphere)$. Denote $\calP_{p, q}(s) = \left\{ \mu \in \calP(\R^d)~:~\int_{\Sphere} \E_\mu[|\theta^\top x|^q]^{p/q} \dd \pi(\theta) \leq s \right\}$. Let $\mu, \nu \in \calP_{p,q}(s)$. Then, there exists a constant $C(p, q) > 0$ depending on $p, q$ such that,
\begin{equation}
    \E| \swdp{\mu_n}{\nu_n}{\pi} - \swdp{\mu}{\nu}{\pi} | \leq C(p, q) s \log(n)^{1/2} n^{-1/2} 
\end{equation}
\end{theorem}

We show that under the sub-Gaussian or the Bernstein moment condition assumptions, 
the assumptions in \Cref{thm:manole} are satisfied, thus allowing its application in these two settings. This yields \Cref{cor:sample_comp_subg,cor:sample_comp_bern}, which we state and prove hereafter. 

\begin{corollary} \label{cor:sample_comp_subg}
    Let $\mu, \nu \in \calP(\R^d)$ and $\rho \in \calP(\Sphere)$. Assume that $\mu$ (respectively, $\nu$) is sub-Gaussian with variance proxy $\sigma^2$ (resp., $\tau^2$). Let $\sigma^2_\star = \max(\sigma^2,\tau^2)$. Then, there exists $C'(p) > 0$ such that,
    \begin{equation}
    \E| \swdp{\mu_n}{\nu_n}{\rho} - \swdp{\mu}{\nu}{\rho} | \leq C'(p) (4\sigma^2_\star)^p \log(n)^{1/2} n^{-1/2} \eqsp.
    \end{equation}
\end{corollary}

\begin{proof}
    Under the sub-Gaussian assumption on $\mu$ and $\nu$, 
    the moments of $\thss \mu, \thss \nu$ can be bounded for any $\theta \in \Sphere$ as follows: for any $k \in \nsets$,
    \begin{equation}
        \E_{\mu} [|\ps{\theta}{x}|^{2k}] \leq k!(4\sigma^2)^k \eqsp, \;\;\; \E_{\nu} [|\ps{\theta}{y}|^{2k}] \leq k!(4\tau^2)^k \eqsp.
    \end{equation} 
    We conclude that $\mu, \nu \in \calP_{p,2(p+1)}(s)$ with $s = \{(p+1)!\}^{p/(2(p+1))}(4\sigma_\star^2)^p$ and $\sigma^2_\star = \max(\sigma^2,\tau^2)$. The final result follows from applying \Cref{thm:manole}.
    
\end{proof}

\begin{corollary} \label{cor:sample_comp_bern}
    Let $\mu, \nu \in \calP(\R^d)$ and $\rho \in \calP(\Sphere)$. Assume that $\mu$ and $\nu$ satisfy the Bernstein condition, with parameters $(\sigma^2, b)$ and $(\tau^2, c)$ respectively. Let $\sigma^2_\star = \max(\sigma^2,\tau^2)$ and $b_\star = \max(b,c)$. Then, there exists $C'(p, q) > 0$ such that
    \begin{equation}
    \E| \swdp{\mu_n}{\nu_n}{\rho} - \swdp{\mu}{\nu}{\rho} | \leq C'(p, q) \sigma_\star^{2p/q} b_\star^{p(q-2)/q} \log(n)^{1/2} n^{-1/2} \eqsp.
    \end{equation}
\end{corollary}

\begin{proof}
    Under the Bernstein condition on the moments of $\mu, \nu$, we can use the definition of the push-forward measures along with the Cauchy-Schwarz inequality and obtain for any $\theta \in \Sphere$ and $k \in \nsets$, 
    \begin{equation}
        \E_{\mu} [|\ps{\theta}{x}|^{2k}] \leq \sigma^2 k! b^{k-2} / 2 \eqsp, \;\;\; \E_{\nu} [|\ps{\theta}{y}|^{2k}] \leq \tau^2 k! c^{k-2} / 2 \eqsp. \label{eq:pushmomentbernstein}
    \end{equation} 
    Let $q > 2p$. By \eqref{eq:pushmomentbernstein}, $\mu, \nu \in \calP_{p,q}(s)$ with $s = (\sigma_\star^2 q! / 2)^{p/q} b_\star^{p(q-2)/q}$. The application of \Cref{thm:manole} concludes the proof.
    
\end{proof}

We can finally provide the refined bounds, assuming the distributions are either sub-Gaussian or satisfy the Bernstein condition. 
On the one hand, incorporating \Cref{prop:momgenfn_subg} 
and \Cref{cor:sample_comp_subg} in \Cref{thm:generic_bound} 
gives us the following corollary. 

\begin{corollary} 
    Let $\mu, \nu \in \calP(\R^d)$. Assume $\mu$ (resp., $\nu$) is sub-Gaussian with variance proxy $\sigma^2$ (resp., $\tau^2$). Let $\sigma_\star^2 \doteq \max(\sigma^2, \tau^2)$. Let $\rho_0 \in \calP(\Sphere)$ and $\delta > 0$. Then, with probability at least $1 - \delta$, for all $\rho \in \calP(\Sphere)$ and $\lambda > 0$, there exists $C > 0$ such that
    \begin{align} 
        \swdone{\mu_n}{\nu_n}{\rho} \leq \swdone{\mu}{\nu}{\rho} &+ \left\{\kld{\rho}{\rho_0} + \log(1/\delta)\right\} \lambda^{-1} \nonumber \\
        &+ \lambda (\sigma^2 + \tau^2) n^{-1} + C \sigma_\star^2 \log(n)^{1/2}n^{-1/2} \eqsp.
    \end{align}
\end{corollary}

On the other hand, we leverage \Cref{prop:momgenfn_bernstein}, \Cref{cor:sample_comp_bern} and \Cref{thm:extension_generic_bound} to derive the specified bound below.

\begin{corollary} 
    Let $\mu, \nu \in \calP(\R^d)$. Assume that $\mu$ and $\nu$ satisfy the Bernstein condition, with parameters $(\sigma^2, b)$ and $(\tau^2, c)$ respectively. Let $\sigma_\star^2 = \max(\sigma^2, \tau^2)$ and $b_\star = \max(b,c)$. Let $\rho_0 \in \calP(\Sphere)$ and $\delta > 0$. Then, with probability at least $1 - \delta$, for all $\rho \in \calP(\Sphere)$ and $\lambda > 0$ s.t. $\lambda < (2b_\star)^{-1} n$, for $q > 2$, there exists $C(q) > 0$ such that
    \begin{align} 
        \swdone{\mu_n}{\nu_n}{\rho} \leq \swdone{\mu}{\nu}{\rho} &+ \left\{\kld{\rho}{\rho_0} + \log(1/\delta)\right\} \lambda^{-1} \nonumber \\
        &+ 2 \lambda \sigma_\star^2 (1-2b_\star\lambda n^{-1})^{-1} n^{-2} + C(q) \sigma_\star^{2/q} b_\star^{(q-2)/q} \log(n)^{1/2}n^{-1/2} \eqsp.
    \end{align}    
\end{corollary}

\section{Additional Experimental Details} \label{appendix:additional_exp}

All our numerical experiments presented in \Cref{sec:expes}
can be reproduced using the code we provided in \href{https://github.com/rubenohana/PAC-Bayesian_Sliced-Wasserstein}{https://github.com/rubenohana/PAC-Bayesian\_Sliced-Wasserstein}.

\subsection{Details on the Algorithmic Procedure}

For clarity, we specify \Cref{alg:PACSW-algorithm} 
when the optimization is performed over the space of von Mises-Fisher distributions (\Cref{def:vmf}). 
The procedure is detailed in \Cref{alg:vmf_sw}.

\begin{algorithm}[ht!]
    \caption{PAC-Bayes bound optimization for vMF-based SW} \label{alg:vmf_sw}%
    \begin{itemize}[label={}, leftmargin=3.35mm, itemsep=0mm]
        \item \textbf{Input:} Datasets: $x_{1:n} = (x_i)_{i=1}^n$, $y_{1:n} = (y_i)_{i=1}^n$
        \item \hphantom{\textbf{Input:}} SW order, number of slices: $p \in [1, \pinf)$, $n_S \in \nsets$
        \item \hphantom{\textbf{Input:}} Bound parameter: $\lambda \in \R^*_+$
        \item \hphantom{\textbf{Input:}} Number of iterations, learning rate: $T \in \nsets$, $\eta \in (0,1)$
        \item \hphantom{\textbf{Input:}} Initialized parameters: $(\mathrm{m}^{(0)}, \kappa^{(0)}) \in \Sphere \times \R^*_+$  
    \end{itemize} 
    \textbf{Output:} Final parameters: $(\mathrm{m}^{(T)}, \kappa^{(T)})$
    \begin{algorithmic}
        \FOR{$t \gets 0$ to $T-1$}
            \STATE $\rho^{(t)} \gets \vMF(\mathrm{m}^{(t)}, \kappa^{(t)})$
            \FOR{$k \gets 1$ to $n_S$}
                \STATE $\theta^{(t)}_k \sim \rho^{(t)}$ \citep[Algorithm 1]{davidson2018}
            \ENDFOR
            \STATE $\rho^{(t)}_n \gets n_S^{-1} \sum_{k=1}^n \updelta_{\theta^{(t)}_k}$
            \STATE $\calL(x_{1:n}, y_{1:n}, \rho^{(t)}, \lambda) \gets \swdp{\mu_n}{\nu_n}{\rho_n^{(t)}} - \lambda^{-1} \kld{\rho^{(t)}}{\rho^{(0)}}$ 
            \STATE $\begin{bmatrix} \mathrm{m}^{(t+1)} \\ \kappa^{(t+1)} \end{bmatrix}\, \gets \, \begin{bmatrix} \mathrm{m}^{(t)} \\ \kappa^{(t)} \end{bmatrix}\, +\, \eta \begin{bmatrix} \nabla_{\mathrm{m}} \calL(x_{1:n}, y_{1:n}, \rho^{(t)}, \lambda) \\ \nabla_\kappa \calL(x_{1:n}, y_{1:n}, \rho^{(t)}, \lambda) \end{bmatrix}$ 
        \ENDFOR
        \STATE \textbf{Return} $(\mathrm{m}^{(T)}, \kappa^{(T)})$
    \end{algorithmic}
\end{algorithm}

\subsection{Additional Results}\label{secapp: generated MNIST}

\Cref{fig:digits} displays additional qualitative results for the generative modeling experiment.
We observe that the images generated by DSW have a better quality than the ones produced by maxSW, even if DSW is not optimized at every training iteration. 

\begin{figure}[t!]
    \centering
    ~\hfill
    \includegraphics[width=3cm]{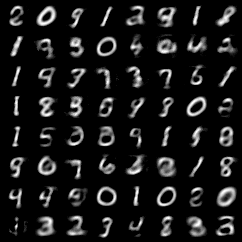}
    ~\hfill~
    \includegraphics[width=3cm]{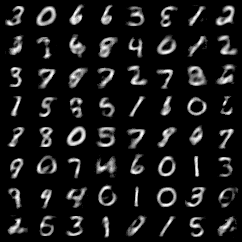}
    ~\hfill~
    \includegraphics[width=3cm]{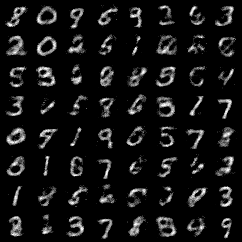}
    ~\hfill~
    \includegraphics[width=3cm]{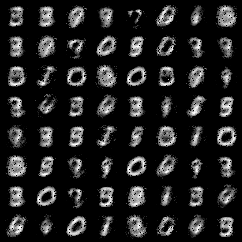}
    \hfill~
    \caption{Examples of generated MNIST digits. Left to right: DSW, DSW-10, maxSW, maxSW-10.}
    \label{fig:digits}
\end{figure}
On \Cref{fig:genmodelingPACSW} are shown the results obtained on the generative modeling experiment of \Cref{sec:expes} using the PAC-SW loss. PAC-SW can be competitive with DSW, but takes more time to execute as the computation of the KL cost is more costly than the regularization term of DSW. However, we observe that the distribution of slices that we learn generalizes well.

\begin{figure}[t!]
    \centering
    \includegraphics[width= 0.5\linewidth]{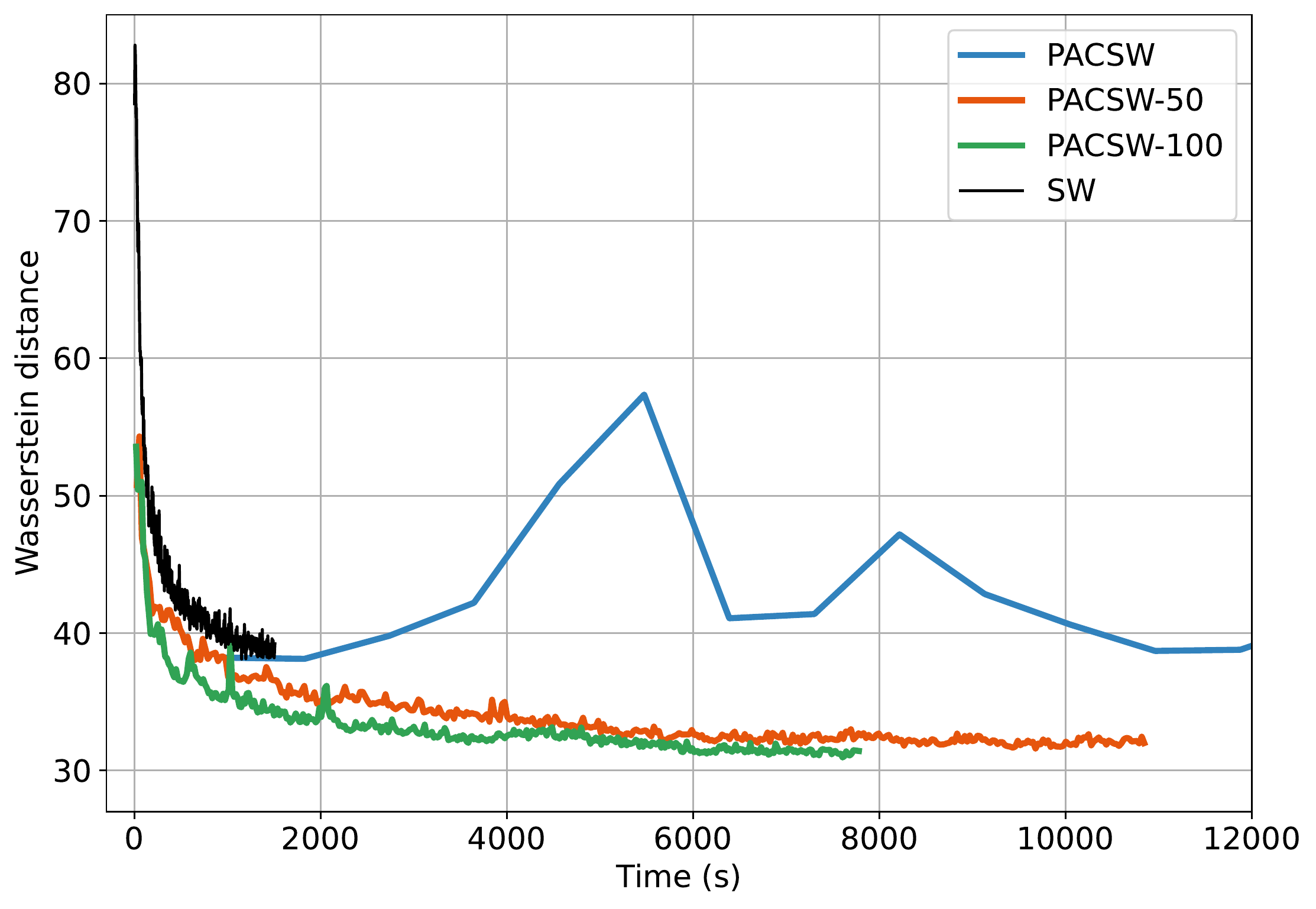}
    \caption{Generative modeling experiment when the slice distribution of PAC-SW is updated either at each iteration (PACSW), every 50 iterations (PACSW-50) or every 100 iterations (PACSW-100). Timing results of this experiment were obtained with a NVIDIA GPU A100 80 GB, compared to \Cref{fig:genmodevolution} which was on a NVIDIA V100.}
    \label{fig:genmodelingPACSW}
\end{figure}



\end{document}